\documentclass[10pt, a4paper, oneside, reqno]{amsart}
\input{style2}

\title[A Distributionally Robust Approach to Fair Classification]{A Distributionally Robust Approach to Fair Classification}
\date{}
\author{Bahar Taskesen, Viet Anh Nguyen, Daniel Kuhn, Jose Blanchet}
\thanks{The authors are with the Risk Analytics and Optimization Chair, EPFL, Switzerland (\texttt{bahar.taskesen, daniel.kuhn@epfl.ch}) and the Department of Management Science and Engineering, Stanford University (\texttt{viet-anh.nguyen, jose.blanchet@stanford.edu}).}

\begin{document}
	
	\begin{abstract}
	We propose a distributionally robust logistic regression model with an unfairness penalty that prevents discrimination with respect to sensitive attributes such as gender or ethnicity. This model is equivalent to a tractable convex optimization problem if a Wasserstein ball centered at the empirical distribution on the training data is used to model distributional uncertainty and if a new convex unfairness measure is used to incentivize equalized opportunities. We demonstrate that the resulting classifier improves fairness at a marginal loss of predictive accuracy on both synthetic and real datasets. We also derive linear programming-based confidence bounds on the level of unfairness of any pre-trained classifier by leveraging techniques from optimal uncertainty quantification over Wasserstein balls. 
	\end{abstract}
	
	\maketitle
	
	\section{Introduction}

    Machine learning algorithms are increasingly used to support human decision-making in sensitive domains and may impact, for example, which individuals will receive jobs, loans, medication, bail or parole. There are several reasons to believe that algorithms will make better decisions than human beings because they are capable of handling much more data than any human may grasp and because they can perform complex computations much faster than humans. In addition, human decisions are often subjective and prone to biases. 
    
    Although algorithmic decision processes are arguably efficient and make good use of all available data, they are not always as objective as one would expect. For example, recent studies have revealed that an algorithm used by the United States criminal justice system had falsely predicted that African Americans are twice as likely to engage in criminality as white Americans \cite{chouldechova2017fair, ref:propublica}. Also, it was recently discovered that a hiring system used by Amazon AI was discriminating against female candidates for software development and technical positions \cite{ref:dastin2018amazon}. In addition, it was shown that Google’s ad-targeting algorithm had proposed higher-paying executive jobs more often to men than to women \cite{ref:datta2015automated}.
    
    There are several possible explanations for biased behaviour of machine learning algorithms. First, the training data could already be corrupted by human biases due to biased device measurements or historically biased human decisions, amongst others. 
    Machine learning algorithms are designed to learn and preserve these biases \cite{ref:buolamwini2018gender,  ref:manrai2016genetic}. Second, minimizing the average prediction error privileges the majority populations over the minorities. Third, sensitive attributes can have an implicit detrimental effect on the decision making process even if they are not explicitly represented in the training data. Sensitive attributes are any attributes such as the race, gender or age of a person that distinguish privileged from unprivileged individuals. It is often illegal to use these sensitive attributes for decision making. Thus, a na\"ive approach to mitigate algorithmic biases would be to remove all sensitive information from the training data. This leads to {\em fairness through unawareness}. However, sensitive attributes are often correlated to other attributes that seem less problematic (such as a person's hair length or skin pigmentation), and this enables algorithms to make unfair recommendations based on predictions of the sensitive attributes. Ultimately, this results in an implicit use of the sensitive attributes under the guise of fairness  \cite{ref:barocas2016big, ref:black2020fliptest, ref:kleinberg2018algorithmic, ref:lipton2018does}.
    
    The scientific community has spent substantial efforts to establish  mathematical definitions of algorithmic fairness and to ensure that machine learning models are actually fair in the sense of these definitions. In the following, we explain some of the most popular fairness definitions in the context of binary classification and identify without loss of generality the positive outcome with the “advantaged” outcome, such as “admission to a college” or “receiving a promotion.” {\em Demographic parity} \cite{ref:dwork2012fairness} requires the likelihood of a positive outcome ({\em e.g.}, a person being hired) to be the same regardless of whether the person is in the protected ({\em e.g.}, female) group or not. {\em Equalized odds}~\cite{ref:hardt2016equality}, which is also referred to as {\em disparate treatment}~\cite{ref:zafar2017fairness}, requires the probability of a person in the positive class being correctly classified
    and the probability of a person in a negative class being misclassified should both be the same for persons in the privileged and unprivileged groups. {\em Equalized opportunities}~\cite{ref:hardt2016equality} can be viewed as a relaxation of the equalized odds criterion as it requires non-discrimination only within the privileged group. Hence, equalized opportunities requires the true positive rates to be equal in the privileged and unprivileged groups. Other notions of fairness include the {\em disparate impact}~\cite{ref:feldman2015certifying} and {\em disparate mistreatment}~\cite{ref:zafar2017fairness} criteria. The central idea behind any notion of fairness is to require the decisions of a classifier to be balanced among the privileged and unprivileged groups and label sets. For a comprehensive survey and further discussions of fairness in machine learning we refer to \cite{ref:berk2018fairness, ref:chouldechova2020snapshot, ref:corbett2017algorithmic, ref:mehrabi2019survey}.

    Logistic regression is one of the most popular classification methods \cite{ref:hosmer2013applied}. Its objective is to establish a probabilistic relationship between a random feature vector $X\in \mc X = \R^p$ and a random binary explanatory variable $Y \in \mc Y=\{0, 1\}$. We assume here that there is a single sensitive attribute $A\in\mc A=\{0,1\}$, which is also random and is {\em not} contained in the feature vector~$X$, and we consider the privileged learning setting \cite{ref:vapnik2009new, ref:quadrianto2017recycling}, where the sensitive information is only available at the training stage but not at the testing stage. Note that predicting $Y$ from $X$ ensures fairness through unawareness. In the remainder, we denote by $\{(\hat x_i, \hat a_i, \hat y_i)\}_{i=1}^N$ a finite set of training samples that are drawn independently from the probability distribution $\PP$ of the joint random vector $(X,A,Y)$. In logistic regression, the conditional probability $\PP[Y=1|X=x]$ is modeled as the sigmoidal hypothesis
    \[
    h_\beta(x) = [{1 + \exp(-\beta^\top x)}]^{-1},
    \]
    where the weight vector $\beta \in \R^p$ constitutes an unknown regression parameter. Classical logistic regression determines $\beta$ by solving the tractable convex optimization problem
    \be \label{eq:LR-standard}
    \Min{\beta}~\frac{1}{N}\sum\limits_{i=1}^N  \ell_\beta(\wh x_i, \wh y_i),\quad \ell_\beta(x,y)= -y \log (h_\beta(x)) - (1-y) \log (1 - h_\beta(x))
    \ee
    which minimizes the empirical \textit{log-loss}, that is, the negative log-likelihood function of the training data.
    To make logistic regression fair, we will include an {\em unfairness measure} in problem~\eqref{eq:LR-standard}. Specifically, we will either include a fairness constraint that requires the unfairness measure to fall below a given threshold, or we will include the unfairness measure as a penalty term in the objective function. As it is not possible to satisfy multiple notions of fairness simultaneously \cite{ref:berk2018fairness, kleinberg2016inherent}, we focus on unfairness measures related to equalized opportunities. However, our method is general enough to cater for other notions of fairness.
    
\begin{definition}[Unfairness measure]
\label{def:unfair}
If $f: [0,1] \to \R$ is  measurable, then the unfairness of a hypothesis $h:\mc X\rightarrow[0,1]$ with respect to $f$ under a distribution $\QQ$ of $(X, A, Y)$ is
\[
    \mathds U_f (\QQ, h) = \big| \EE_{\QQ}[ f(h(X)) | A = 1, Y = 1] - \EE_{\QQ}[ f(h(X)) | A = 0, Y = 1] \big|.
 \]
\end{definition}
The larger $\mathds U_f(\QQ, h)$, the more unfair is the hypothesis $h$, and if $\mathds U_f(\QQ, h)=0$, then the hypothesis is maximally fair. Different choices of~$f$ induce different notions of fairness.
 If $f(z) = \mathbbm{1}_{\{z \ge \tau\}}$, then $\mathds U_f(\QQ, h)=0$ means that $h$ is fair in view of the {\em equalized opportunities} criterion \cite{ref:hardt2016equality}. Here, $\tau \in [0, 1]$ is the classification threshold. 
 If $f(z) = z$, then $\mathds U_f(\QQ, h)=0$ means that the hypothesis $h$ is fair in view of the {\em probabilistic equalized opportunities} criterion for probabilistic classifiers \cite{ref:pleiss2017fairness}.
 
It is well known that increasing the fairness of an algorithm typically reduces its accuracy \cite{ref:friedler2019comparative, ref:lipton2018does, ref:menon2018cost}. This prompts us to introduce an ideal {\em fair} logistic regression model
\be \label{eq:LR-true-all-f}
\Min{\beta}~\EE_{ \PP}[- Y \log (h_\beta(X)) - (1-Y) \log (1 - h_\beta(X))] + \eta \mathds U_f( \PP, h_\beta),
\ee
where $\eta \in \R_+$ is a tuning parameter that balances the trade-off between accuracy and fairness. Unfortunately, problem~\eqref{eq:LR-true-all-f} is difficult to solve for several reasons. If $f(z)=\mathbbm{1}_{\{z \ge \tau\}}$, then the unfairness measure $\mathds U_f( \PP, h_\beta)$ is discontinuous in $\beta$, and if $f(z)=z$, then $\mathds U_f( \PP, h_\beta)$---though smooth---is still non-convex in $\beta$. In both cases, it seems difficult to solve~\eqref{eq:LR-true-all-f} to global optimality. In addition, the distribution $\PP$ is unknown and only indirectly observable through the $N$ independent training samples. Thus, an important input for problem~\eqref{eq:LR-true-all-f} is unavailable in practice. The latter shortcoming could be addressed by simply replacing the unknown true distribution $\PP$ in \eqref{eq:LR-true-all-f} with the empirical distribution $\Pnom_N$, which is defined as the discrete uniform distribution on the $N$ training samples. However, this na\"ive approach could result in over-fitting and yield classifiers with a poor out-of-sample performance (both in terms of accuracy and fairness) if $N$ is small relative to~$p$.

The concerns over poor out-of-sample performance prompt us to pursue a {\em distributionally robust} approach, whereby the objective function in~\eqref{eq:LR-true-all-f} is minimized in view of the most adverse distribution $\QQ$ within some {\em ambiguity set} that reflects all available distributional information. The ambiguity set could be characterized through moment and support information \cite{ref:delage2010distributionally, ref:goh2010distributionally, ref:wiesemann2014distributionally}, or it could be defined as a ball around $\Pnom_N$ with respect to a distance measure for distributions such as the Prohorov metric \cite{ref:erdougan2006ambiguous} or the Kullback-Leibler divergence \cite{ref:hu2013kullback}. Due to its attractive measure concentration properties, we use here the Wasserstein metric to construct ambiguity sets \cite{ref:kuhn2019wasserstein, ref:esfahani2018data, ref:pflug2007ambiguity}. Moreover, Wasserstein distributional robustness offers probabilistic interpretations for popular regularization techniques \cite{ref:blanchet2019robust, ref:gao2017wasserstein, ref:shafieezadeh2019regularization,  ref:abadeh2015distributionally}.

The main contributions of this paper can be summarized as follows.
\begin{enumerate}[leftmargin=5mm]
    \item \textbf{Log-probabilistic equalized opportunities:} We propose a new unfairness measure and the corresponding fairness criterion, termed log-probabilistic equalized opportunities, which approximates the probabilistic equalized opportunities criterion. 
    We then prove that the empirical ({\em i.e.}, $\PP=\Pnom_N$) fair logistic regression model~\eqref{eq:LR-true-all-f}  with the new unfairness measure is equivalent to a tractable convex program.
    \item \textbf{Distributionally robust fair logistic regression:} We robustify the fair logistic regression model against all distributions in a Wasserstein  ball centered at~$\Pnom_N$, and we prove that this model is still equivalent to a tractable convex program if unfairness is quantified under the log-probabilistic equalized opportunities criterion. Experiments suggest that the resulting classifiers improve fairness at a marginal loss of accuracy.
    \item \textbf{Unfairness quantification:} Using similar  techniques from Wasserstein distributionally robust optimization, we develop two highly tractable linear programs whose optimal values provide confidence bounds on the unfairness of {\em any}  fixed classifier with respect to the (classical) probabilistic equalized opportunities criterion. We also devise a hypothesis test that checks whether a given classifier is fair in view of equalized opportunities.
\end{enumerate}

The existing literature on algorithmic fairness can be subdivided into three categories. Papers in the first category propose to pre-process the training data before solving a plain-vanilla classification problem \cite{ref:calmon2017optimized, ref:del2018obtaining, ref:feldman2015certifying, ref:kamiran2012data, ref:luong2011k, ref:samadi2018price, ref:zemel2013learning}. Papers in the second category enforce fairness during the training step by appending fairness constraints to the classification problem \cite{ref:donini2018empirical, ref:menon2018cost, ref:woodworth2017learning, ref:zafar2017fairness, ref:zafar2015fairness}, by including regularization terms that penalize discrimination \cite{ref:baharlouei2019r,
ref:huang2019stable,
ref:kamishima2012fairness, ref:kamishima2011fairness} or by (approximately) penalizing any mismatches between the true positive rates and the false negative rates across different groups \cite{ref:bechavod2017penalizing}. Several other papers in this category propose adversarial approaches to algorithmic fairness  \cite{ref:edwards2015censoring, ref:garg2019counterfactual, ref:hashimoto2018fairness, ref:kannan2018adversarial, ref:madras2018learning,  ref:rezaei2020fairness, ref:yurochkin2020training, ref:zhang2018mitigating}. Papers in the third category modify a pre-trained classifier in order to increase its fairness properties while preserving its classification performance as much as possible \cite{ref:corbett2017algorithmic, ref:dwork2018decoupled, ref:hardt2016equality, ref:menon2018cost}. 
    
The method proposed here can be viewed as an adversarial approach pertaining to the second category. There are only few other papers that study fairness from a distributionally robust perspective. A classification model with fairness constraints embedded in the ambiguity set is proposed in \cite{ref:rezaei2020fairness}, a repeated loss minimization model with a $\chi^2$-divergence ambiguity set is considered in~\cite{ref:hashimoto2018fairness} and robust fairness constraints based on a total variation ambiguity set that captures noisy protected group information is described in~\cite{ref:wang2020robust}. In addition, a fair distributionally robust classification model with a Wasserstein ambiguity set is studied in~\cite{ref:yurochkin2020training}, but this model deals with individual fairness and does not admit a tractable convex reformulation. In contrast, we consider marginally constrained Wasserstein ambiguity sets to enforce a notion of group fairness and provide a tractable convex reformulation.
    
\section{Fair Logistic Regression}
\label{sect:fair}
Recall that the fair logistic regression model~\eqref{eq:LR-true-all-f} is non-convex if $f(z)=\mathbbm{1}_{\{z \geq \tau\}}$, which induces equalized opportunities, or if $f(z) = z$, which induces probabilistic equalized opportunities. In order to convexify~\eqref{eq:LR-true-all-f}, we thus propose a new unfairness measure corresponding to $f(z)=1+\log(z)$, and we refer to the fairness criterion induced by the condition $\mathds U_f(\QQ, h)=0$ as {\em log-probabilistic equalized opportunities}. A classifier is fair in view of this criterion if the expected log-probability of a person in the positive class being correctly classified is the same for persons in the privileged and unprivileged groups. We also note that the log-probability function $f(h_\beta(x))=1-\log(1 + \exp(-\beta^\top x))$ can be viewed as a concave approximation of the sigmoid function $h_\beta(x)$. Concave (or convex) approximations of non-convex functions are routinely used in machine learning and arise, for example, when one replaces a non-convex loss function (such as the zero-one loss) with a convex surrogate loss function (such as the hinge loss or the log-loss) or when one replaces a non-convex risk measure (such as the value-at-risk) with a convex one (such as the conditional value-at-risk).

We now denote by $\wh p_{a y} = \Pnom_N(A = a, Y = y)$ the empirical proportion of people with attribute $a\in\mc A$ in class $y\in\mc Y$, and we define $r_a = 1/ \hat p_{a1}$ for all $a\in  \mc A$. Using this notation, we can prove that the logistic regression model~\eqref{eq:LR-true-all-f} with the log-probabilistic equalized opportunities unfairness measure is tractable under the empirical distribution for all sufficiently small $\eta$.

\begin{theorem}[Fair logistic regression] \label{thm:FLR}
If $f(z) = \log(z)$, $\eta \leq \min{\{ \hat p_{11}, \hat p_{01}\}}$ and $\PP=\Pnom_N$, then problem~\eqref{eq:LR-true-all-f} is equivalent to the tractable convex program
\[
\begin{array}{cl}
    \Min{\beta \in \R^p, t \in \R} & t \\[-0.5ex]
    \st & \EE_{\Pnom_N} [\ell_\beta(X,Y) + \eta r_1 \log(h_\beta(X))\mathbbm 1_{(1,1)}(A,Y) - \eta r_{0} \log(h_\beta(X))\mathbbm{1}_{(0, 1)}(A,Y)] \le t \\
    & \EE_{\Pnom_N} [\ell_\beta(X,Y) + \eta r_0 \log(h_\beta(X))\mathbbm 1_{(0,1)}(A,Y) - \eta r_{1} \log(h_\beta(X))\mathbbm{1}_{(1, 1)}(A,Y)] \le t,
\end{array}
\]
where the expectation under $\Pnom_N$ is a finite sum.
\end{theorem}

\section{Distributionally Robust Fair Logistic Regression}
    \label{sect:training} 
    Approximating the unknown data-generating distribution~$\PP$ with the empirical distribution~$\Pnom_N$ may result in overfitting. Following~\cite{ref:blanchet2019robust, ref:gao2017wasserstein, ref:shafieezadeh2019regularization,  ref:abadeh2015distributionally}, we thus regularize the nominal classification problem under $\Pnom_N$ by robustifying it against all distributions in a Wasserstein ball around $\Pnom_N$ that contains the unknown true distribution $\PP$ with high confidence.
  
    \begin{definition}[Wasserstein distance]
    The type-$1$ Wasserstein distance between two probability distributions $\QQ_1$ and $\QQ_2$ of a random vector $\xi\in\R^n$ is defined as 
    \be \label{eq:kantorovich}
    \Wass(\QQ_1,\QQ_2 ) = \inf_{\pi  \in \Pi(\QQ_1, \QQ_2)} \EE_\pi[ c(\xi_1,\xi_2)],
    \ee 
    where $\Pi (\QQ_1, \QQ_2)$ denotes the set of all joint distributions of the random vectors $\xi_1\in\R^n$ and $\xi_2\in \R^n$ under which $\xi_1$ and $\xi_2$ have marginal distributions $\QQ_1$ and $\QQ_2$, respectively, and where $c:\R^n\times\R^n\rightarrow [0,\infty]$ constitutes a lower semi-continuous ground metric.
   \end{definition}
   
    When computing Wasserstein distances between distributions on $\mc X \times \mc A \times \mc Y$, we will use
    \be
    \label{eq:cost}
        c\big( (x, a, y), (x', a', y') \big) = \| x - x'\| + \kappa_{\mc A} | a - a'| + \kappa_{\mc Y} | y - y'|
    \ee
    as the ground metric, where $\|\cdot\|$ is a norm on $\R^p$ and $\kappa_{\mc A}, \kappa_{\mc Y} \in (0,\infty]$. Using the Wasserstein distance with the ground metric~\eqref{eq:cost}, we define the ambiguity set $\mbb B_\rho(\Pnom_N)$ as the Wasserstein ball of radius $\rho\ge 0$ around the empirical distribution $\Pnom_N$, intersected with the set of all distributions under which the marginal of $(A,Y)$ matches the empirical marginal. Thus,
\[
    \mbb B_\rho(\Pnom_N) = \left\{ \QQ \in \mc M:         \Wass(\QQ , \Pnom_N) \leq \rho, ~
        \QQ(A = a, Y = y) = \hat p_{ay} 
        \quad\forall a \in \mc A,~ y\in \mc Y 
        \right\},
\]
    where $\mc M$ stands for the set of all possible distributions on $\mc X\times\mc A\times\mc Y$. Note that $\mbb B_\rho(\Pnom_N)$ is non-empty as it contains at least $\Pnom_N$. Note also that all distributions in $\mbb B_\rho(\Pnom_N)$ can be obtained by reshaping $\Pnom_N$ at a transportation cost of at most~$\rho$. The parameter~$\kappa_\mc A$ represents the transportation cost of changing the sensitive attribute from $A$ to $1-A$, and thus it can be viewed as our trust in~$A$. A similar interpretation applies to $\kappa_\mc Y$. We can now formally introduce the {\em distributionally robust fair} logistic regression model
\be
    \label{eq:dro_fair_training}
        \min_{\beta}~\Sup{\QQ \in \mbb B_\rho (\hat  \PP_N)}~ \EE_{\QQ}[- Y \log (h_\beta(X)) - (1-Y) \log (1 - h_\beta(X))] + \eta \mathds U_f(\QQ, h_\beta),
    \ee
    which minimizes a combination of the expected log-loss and some unfairness measure under the most adverse distribution in $\mbb B_\rho(\Pnom_N)$. Wasserstein ambiguity sets with marginal constraints were first studied in \cite{ref:frogner2019incorporating}, where it was found that restricting the marginals of the outputs and/or the features eliminates unrealistic data distributions from the ambiguity set and often improves the performance of the resulting classifiers while maintaining strong robustness guarantees. We are now ready to prove that~\eqref{eq:dro_fair_training} is tractable if the log-probabilistic equalized opportunities unfairness measure is used and if $\eta$ is sufficiently small.
  
    \begin{theorem}[Distributionally robust fair logistic regression]
    \label{thm:training_refor}
    If $f(z) = \log(z)$ and $\eta \leq \min\{\hat p_{11},\hat p_{01}\}$, then problem~\eqref{eq:dro_fair_training}
    is equivalent to the tractable convex program
    \be
    \nonumber
    \hspace{-1.75mm}
    \begin{array}{c@{~}ll}
         \min  &t  \\
         \st & \beta \in \R^p, \;t\in \R,\; \lambda_0,
         \lambda_1 \in \R_+,\; \mu_0, \mu_1 \in \R^{|\mc A| \times |\mc Y|},\; \nu_0, \nu_1 \in \R^{N}\\[2ex]
         & \|\beta\|_*(1 + \eta r_{0})\leq \lambda_1 ,\quad\|\beta\|_*(1 + \eta r_{1})\leq \lambda_0 \\
         &\rho \lambda_{a'} + \sum\limits_{a \in \mc A,\,y\in \mc Y} \hat p_{a y} \, 
         \mu_{a' a y} + \frac{1}{N}\sum\limits_{i=1}^N \nu_{a' i} \le t\quad \forall a' \in \{0, 1\}\\
         &\hspace{-2mm} \left.
        \begin{array}{l}
          \log(h_\beta(-\hat x_i)) +\kappa_\mc{A} |a - \hat a_i| \lambda_a +\kappa_{\mc Y}|
         \hat y_i|\lambda_{a} +\mu_{a a 0} +\nu_{a i} \ge 0   \\
          \log(h_\beta(-\hat x_i)) + \kappa_\mc{A} |a'-\hat a_i| \lambda_a +\kappa_{\mc Y}|\hat y_i|\lambda_a +\mu_{a a' 0} +\nu_{ai} \ge 0\\
         (1 - \eta r_a) \log(h_\beta(\hat x_i)) + \kappa_\mc{A} |a-\hat a_i| \lambda_a +\kappa_{\mc Y}|1 -\hat y_i|\lambda_a +\mu_{a a 1}+\nu_{ai}\ge 0 \\
         (1+\eta r_{a'}) \log(h_\beta(\hat x_i)) + \kappa_\mc{A} |a'-\hat a_i| \lambda_a +\kappa_{\mc Y}|1 -\hat y_i|\lambda_a +\mu_{a a' 1}+\nu_{ai} \ge 0 
         \end{array}\hspace{-1.5mm}
         \right\}\!\!\! \begin{array}{l}
              \forall i  \in [N],  \\
              \forall a,a' \in\mc A:\\
              a' = 1-a,
         \end{array}
        \end{array}
    \ee
    where $\|\cdot\|_*$ represents the norm dual to $\|\cdot\|$ on $\R^p$.
    \end{theorem}

    Note that the assumption on $\eta$ implies that $\eta r_a=\eta/\hat p_{a1}\leq 1$ for all $a\in\mc A$, and thus it is easy to verify that the reformulation of Theorem~\ref{thm:training_refor} is indeed convex. For many commonly used norms, this reformulation can be addressed with an exponential cone solver such as MOSEK. Alternatively, one may develop customized first-order methods by adapting the algoritghm proposed in~\cite{ref:li2019first} to account for an unfairness measure in the objective.
   
\section{Unfairness Quantification}
\label{sect:quantify}
A regulator may find it difficult to decide whether or not a given classifier is susceptible to discrimination because this decision may critically dependent on the test data at hand. As a remedy, we develop here a method for quantifying the unfairness of a pre-trained probabilistic classifier $h$ under perturbations of the test distribution, and we propose a systematic approach to decide whether this classifier is fair or not. To this end, we first define the worst (highest) and best (lowest) possible unfairness levels of the classifier $h$ across all distributions in a Wasserstein ambiguity set of the form $\mbb B_\rho(\Pnom_N)$ as 
    \be \nonumber
        \textstyle \overline{\mathds U}_f = \sup_{\QQ \in \mbb B_\rho(\Pnom_N)} \mathds U_f(\QQ, h)\quad \text{ and } \quad \underline{\mathds U}_f = \inf_{\QQ \in \mbb B_\rho(\Pnom_N)} \mathds U_f(\QQ, h),
    \ee
respectively. Here, by slight abuse of notation, $\Pnom_N$ should be interpreted as the discrete uniform distribution on $N$ {\em test samples} $\{(\hat x_i, \hat a_i, \hat y_i)\}_{i=1}^N$ drawn independently from~$\PP$.

The first main result of this section is to show that both $\overline{\mathds U}_f$ and $\underline{\mathds U}_f$ can be re-expressed in terms of the optimal values of two highly scalable linear programs when 
$f(z) = \mathbbm 1_{\{z\ge \tau\}}$, that is, when unfairness is measured with respect to the standard equalized opportunities criterion. Thus, there is no need to resort to approximations involving log-probabilities.

To see this, we define $\mc X_0=\{x\in\mc X:h(x)<\tau\}$ and $\mc X_1=\{x\in\mc X:h(x)\ge\tau\}$, and we set
   \be \label{eq:V_f-def}
    \VV(a, a') = \sup\limits_{\QQ \in \mbb B_\rho(\hat \PP_N)} \QQ [X\in\mc X_1 | A=a, Y=1] -\QQ [X\in\mc X_1| A=a', Y=1]\quad \forall a,a'\in\mc A. 
    \ee
    In addition, we define $d_{yi} = \inf_{x \in \mc X_y} \| x - \wh x_i\|$ for all $y\in\mc Y$ and $i\in[N]$
as the distances of the testing features $\hat x_i$ to the sets $\mc X_y$. Our ability to quantify the fairness of $h$ will critically depend on whether $d_{yi}$ can be computed efficiently. For linear classifiers the sets $\mc X_1$ and $\mc X_0$ constitute half-spaces, and therefore $d_{yi}$ can be computed in closed form. For more complicated classifiers such as neural networks, however, one may have to resort to heuristics to estimate $d_{yi}$. Using this notation, we can state the following main result.

\begin{theorem}[Unfairness quantification] \label{thm:quantify-EO}
If $f(z) = \mathbbm 1_{\{z\ge \tau\}}$, then we obtain $\overline{\mathds U}_f= \max\{\VV(1, 0), \VV(0,1)\}$  and $\underline{\mathds U}_f = - \min\{0, \VV(1, 0), \VV(0, 1)\}$, where $\VV(a,a')$ can be computed for all $a,a'\in\mc A$ with $a\neq a'$ as the optimal value of a tractable linear program, that is,
        \be
        \nonumber 
        \VV(a,a')=\left\{ \begin{array}{cll}
        \min & \rho \lambda + \wh p^\top \mu + N^{-1} \mathbf{1}^\top \nu \\
        \st & \lambda \in \R_+,~\mu \in \R^{2\times 2},~\nu \in \R^N \\
        & \hspace{-2mm} \left.
        \begin{array}{l}
        \nu_i + \kappa_{\mc A} |a - \wh a_i | \lambda + \kappa_{\mc Y} | \wh y_i | \lambda + \mu_{a0} \ge 0 \\
        \nu_i + \kappa_{\mc A} |a' - \wh a_i | \lambda + \kappa_{\mc Y} | \wh y_i | \lambda + \mu_{a'0} \ge 0 \\
        \nu_i + \kappa_{\mc A} |a - \wh a_i | \lambda + \kappa_{\mc Y} |1- \wh y_i | \lambda + \mu_{a1} \ge 0 \\
        \nu_i + d_{1i} \lambda + \kappa_{\mc A} |a - \wh a_i | \lambda + \kappa_{\mc Y} |1- \wh y_i | \lambda + \mu_{a1} \ge r_a \\
        \nu_i +  d_{0i} \lambda + \kappa_{\mc A} |a' - \wh a_i | \lambda + \kappa_{\mc Y} |1- \wh y_i | \lambda + \mu_{a'1} \ge 0\\ 
         \nu_i  + \kappa_{\mc A} |a' - \wh a_i | \lambda + \kappa_{\mc Y} |1- \wh y_i | \lambda + \mu_{a'1} \ge - r_{a'}
         \end{array} \right\} ~\forall i \in [N].
        \end{array}\right.
    \ee
\end{theorem}

The bounds on the unfairness measure related to equalized opportunities can be computed even faster if we have absolute trust in $A$ and $Y$, that is, if $\kappa_\mc A=\kappa_\mc Y=\infty$.
To see this, we select 
$\wh x_i\opt \in \arg\min_{x_i \in \partial\mc X_1} \| x_i - \wh x_i \|$ and we assume for the simplicity of exposition that $\|\hat x_i - \hat x_i^\star\| > 0$ for all $i \in [N]$. 
    We define non-negative rewards and weights through
\[
        (c_{aa'i}, w_{aa'i}) = \begin{cases}
        (r_a, d_{1i}) & \text{if } \wh x_i\in \text{int}(\mc X_{ 0}),\;\wh a_i = a,\; \wh y_i = 1, \\
        (r_{a'}, d_{0i}) &\text{if } \wh x_i\in \text{int}(\mc X_{1}),\;\wh a_i = a',\; \wh y_i = 1, \\ (0, +\infty) & \text{otherwise}
        \end{cases}
\]
for all $a,a'\in\mc A$ and $i\in[N]$. In addition, we introduce the notational shorthand 
\[
\hat \VV(a, a') =  \Pnom_N [X\in\mc X_1 | A=a, Y=1] -\Pnom_N [X\in\mc X_1| A=a', Y=1]\quad \forall a,a'\in\mc A,
\]
which can be evaluated by computing a finite sum. We can then prove the following theorem.

    \begin{theorem}[Absolute trust in $A$ and $Y$] \label{thm:quantify-EO-infinite}
         If $f(z) = \mathbbm 1_{\{z\ge \tau\}}$ and $\kappa_{\mc A} = \kappa_{\mc Y} = \infty$, then
        \begin{equation}
        \label{eq:quantify-EO-infinite-fin}
            \mathds V(a,a') = 
                \hat{\mathds V}(a, a') + \max_{z\in [0,1]^N}\left\{\ds \frac{1}{N}\sum_{i \in [N]} c_{aa'i} z_i \;:\;
            \frac{1}{N}\sum\limits_{i \in [N] } w_{aa'i} z_i\leq \rho \right\}\quad \forall a,a'\in\mc A.
     \end{equation}
    \end{theorem}

    Theorem~\ref{thm:quantify-EO-infinite} asserts that evaluating $\mathds V(a, a')$ is tantamount to solving a continuous knapsack problem in $N$ variables, which can be solved by a greedy heuristics in time $\mc O(N \log N)$.

    It is instructive to study the worst- and best-case distributions that determine~$\overline{\mathds U}_f$ and~$\underline{\mathds U}_f$. By Theorem~\ref{thm:quantify-EO}, these extremal distributions can be constructed from the extremal distributions that determine~$\mathds V(1, 0)$ and~$\mathds V(0,1)$. As the objective function of~\eqref{eq:V_f-def} represents a conditional expectation of a discontinuous integrand that fails to be upper semi-continuous, however, the supremum in~\eqref{eq:V_f-def} is not attained. We thus construct suboptimal distributions that attain the supremum of~\eqref{eq:V_f-def} asymptotically. For linear classifiers, the projections $\hat x^\star_{i}$ of the test samples to the decision boundary may be constructed analytically. For more sophisticated classifiers, however, they may have to be approximated using heuristic methods.
    \begin{proposition}[Extremal distributions]  \label{prop:extreme-infty}
    If $f(z) = \mathbbm 1_{\{z\ge \tau\}}$, $\kappa_{\mc A} = \kappa_{\mc Y} = \infty$ and $z^\star$ is a maximizer of the linear program in~\eqref{eq:quantify-EO-infinite-fin} for some fixed $a, a'\in \mc A$, then 
    \be \notag
    \QQ\opt=\frac{1}{N} \left( 
            \textstyle \sum_{i=1}^N z_i\opt \delta_{(\wh x_i\opt, \wh a_i, \wh y_i)} + \sum_{i=1}^N (1-z_i\opt) \delta_{(\wh x_i, \wh a_i, \wh y_i)}
        \right),
    \ee
    is feasible in~\eqref{eq:V_f-def}, and for any $\eps >0$ there exists $\QQ\opt_\eps\in \mbb B_\eps(\QQ\opt)$ that is $\eps$-suboptimal in~\eqref{eq:V_f-def}.
    \end{proposition} 
    Note that $\QQ\opt$ is in general strictly suboptimal in~\eqref{eq:V_f-def}, but every neighborhood of $\QQ\opt$ contains $\eps$-suboptimal distributions $\QQ\opt_\eps$ for any $\eps>0$. In principle, $\QQ\opt_\eps$ can be constructed explicitly from $\QQ\opt$. However, the construction is cumbersome and therefore omitted. 
    
    The unfairness quantification procedure of this section can be used to devise a hypothesis test that checks whether a given classifier $h$ is fair with respect to the equalized opportunities criterion. By definition, $h$ is fair if $\mathds U_f(\PP, h) = 0$, where $f(z) = \mathbbm 1_{\{z\ge \tau\}}$ and~$\PP$ is the unknown true distribution of $(X, A, Y)$. If $\mc F = \{ \QQ \in \mc M: \mathds U_f(\QQ, h) = 0\}$ represents the family of all distributions under which~$h$ is fair, then testing for fairness is equivalent to testing whether the true distribution~$\PP$ belongs to~$\mc F$. This can be expressed formally as a hypothesis testing problem with the null hypothesis $\mathrm{H}_0:\PP \not\in \mc F$ and the alternative hypothesis $\mathrm{H}_1:\PP \in \mc F$. Given the empirical distribution $\Pnom_N$ on the test data, the proposed hypothesis test rejects~$\text{H}_0$ whenever $\wh \rho = \inf_{\QQ \in \mc F} \Wass(\Pnom_N, \QQ) > s$, 
    where $s$ represents a test statistic. The distance between $\Pnom_N$ and $\mc F$ can be expressed as $\wh \rho = \inf\{ \rho: \inf_{\QQ \in \mbb B_\rho(\Pnom_N)} \mathds U_f(\QQ, h) = 0\}$, 
    where the outer minimization problem can be solved efficiently by bisection over $\rho \ge 0$, while the inner unfairness quantification problem can be solved by the linear programming techniques developed in this section. It remains to compute the test statistic $s$, which could be obtained by a subsampling procedure~\cite{ref:politis1999subsampling}. We leave this for future research.
    \section{Numerical Experiments}
    \label{sec:numerical}
     Below we refer $\mathds U_f(\PP, h)$ as the deterministic unfairness (Det-UNF) if $f(z) = \mathbbm{1}_{\{z \geq \tau\}}$, the probabilistic unfairness (Prob-UNF) if $f(z) = z$ and the log-probabilistic unfairness (LogProb-UNF) if $f(z) = 1+\log(z)$. Details regarding the setup of the experiments such as the data generation procedure and parameter selection etc.\ are relegated to Appendix~\ref{sect:further_disc_numerical}.
 
  \begin{wrapfigure}{r}{.38\textwidth}
  \vspace{-.8cm}
    \begin{minipage}{\linewidth}
    \centering\captionsetup[subfigure]{justification=centering}
    \includegraphics[width=.95\linewidth]{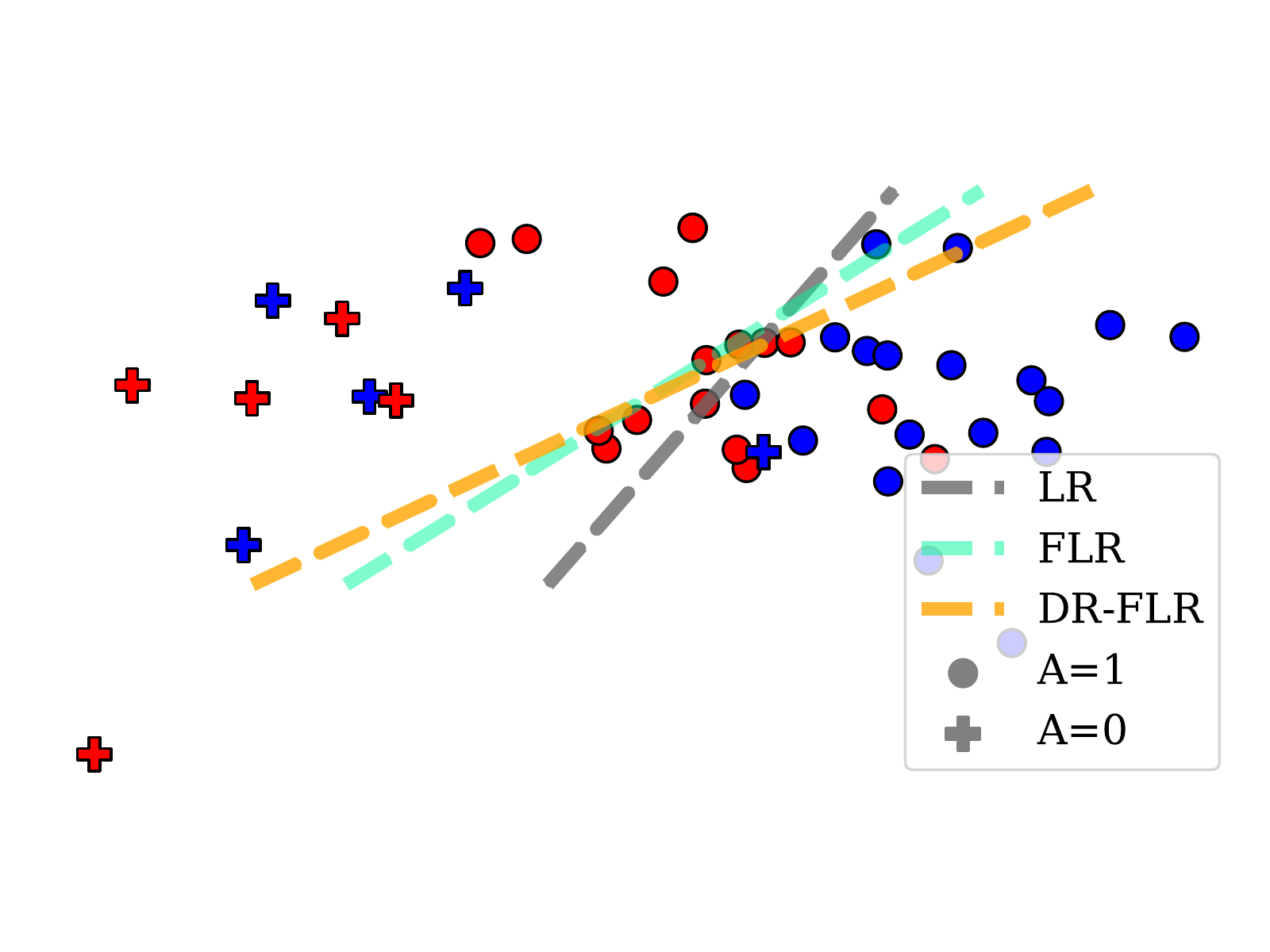}
    \vspace{-0.7cm}\par\vfill
  \includegraphics[width=.95\linewidth]{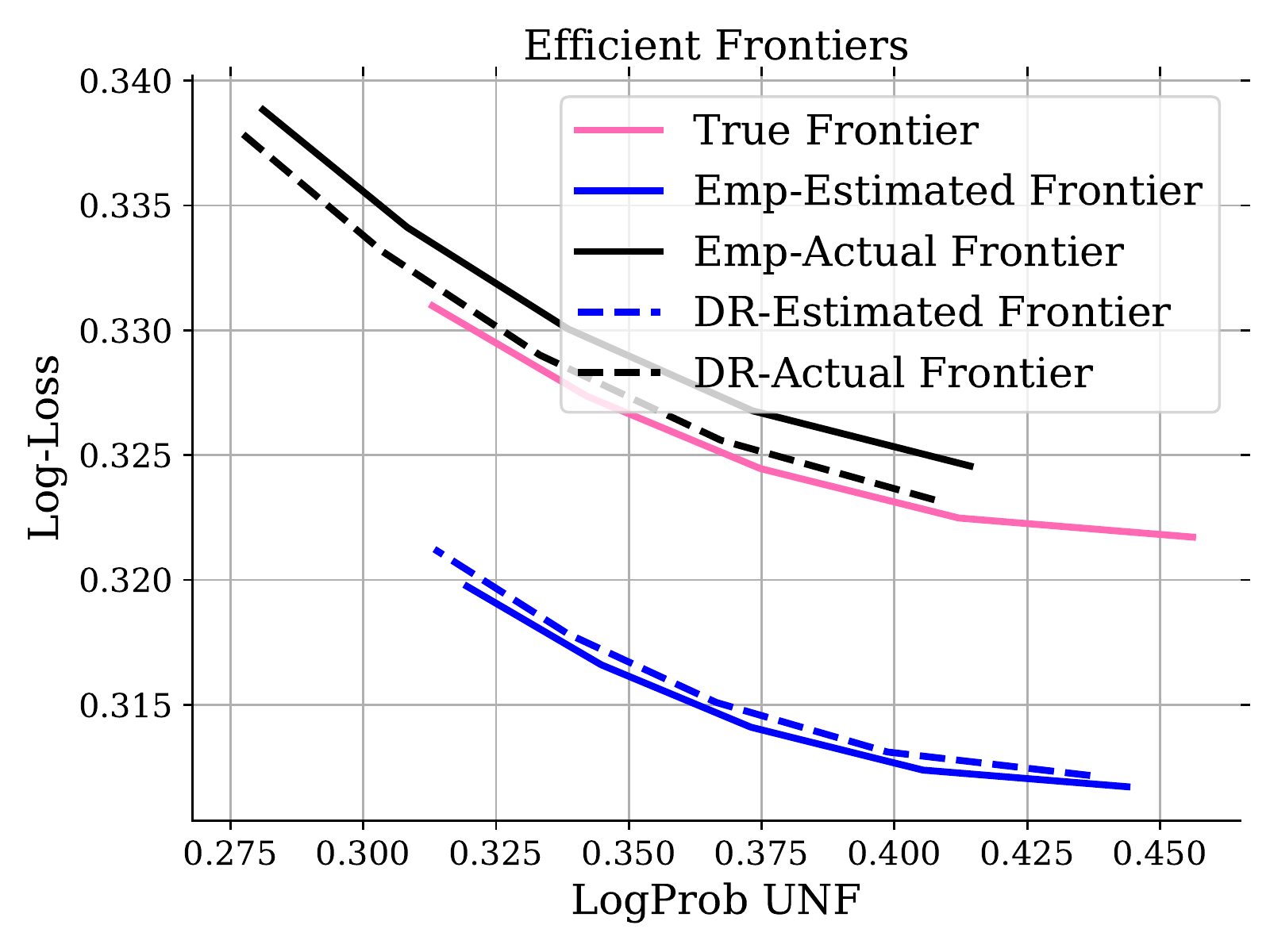}
  \vspace{-.1cm}
  \end{minipage}
  \vspace{-.1cm}
\caption{Classification boundaries (top), Pareto Frontiers (bottom)}\label{fig:decision_boundaries_frontiers}
\vspace{-.6cm}
\end{wrapfigure}
 \noindent\textbf{Synthetic Experiments.} To show the effects of the unfairness penalty and the robustification, we compare the classical, fair and distributionally robust fair logistic regression models (LR, FLR and DR-FLR, respectively) on a dataset with $N=25$ training samples and $p=2$ features. As the sensitive attribute $A$ strongly correlates with $X_1$, fair classifiers assign low weight to $X_1$, which leads to horizontal decision boundaries. Penalizing unfairness with $\eta=0.1$ and robustifying the model with a Wasserstein radius of $\rho = 0.05$ ostensibly increases the fairness of the classifier, see Figure~\ref{fig:decision_boundaries_frontiers} (top). Compared to the LR classifier, the DR-FLR classifier lowers Det-UNF from 0.86 to 0.58 at the expense of reducing the accuracy from $69\%$ to $62\%$.

    The fair logistic regression model~\eqref{eq:dro_fair_training} constitutes a bi-criteria optimization problem that simulataneously minimizes the log-loss and the log-probabilistic unfairness. 
    It is thus reminiscent of the Markowitz mean-variance model that seeks an optimal trade-off between the risk and return of an investment portfolio. The optimal classifers for different values of~$\eta$ trace out a Pareto frontier in the unfairness/loss plane. 
    Following~\cite{ref:broadie1993computing}, we can now distinguish true, estimated and actual Pareto frontiers. The true frontier is obtained by training and evaluating the classifier under the (unknown) true distribution, while the estimated and actual frontiers are obtained by training the classifier on the training dataset and evaluating it on the training and testing datasets, respectively. It is known that the estimated frontier optimistically underestimates and the actual frontier pessimistically overestimates the true frontier on average \cite{ref:broadie1993computing}. It has also been argued that robustifying a bi-criteria model tends to move the actual and estimated frontiers closer to each other as well as closer to the true frontier \cite{ref:martin2010robust}, thus improving out-of-sample performance.
    Figure~\ref{fig:decision_boundaries_frontiers} (bottom) visualizes this effect for a synthetic dataset, where the sensitive attributes correlate with the labels.
\begin{center}
\vspace{-0.7cm}
\begin{table*}
\setlength{\tabcolsep}{0.05cm}
\tiny
\centering
\begin{tabular}{|l|l|c|c|c|c|c|c|}
\hline
Dataset & Metric
& {LR} 
&{FLR} 
&{DOB${}^+$\cite{ref:donini2018empirical}}&{ZVRG \cite{ref:zafar2017fairness}}
&{DR-FLR}\\
\hline
\hline
\multirow{4}{*}{Drug}
    & Accuracy& 
    $0.78 {\pm} 0.01$&
    $0.78 \pm 0.01$ &
    $0.78 \pm 0.01$ &
    $\mathbf{0.79} \pm \mathbf{0.01}$&
    $ 0.78 \pm 0.00$\\
    & Det-UNF &
    $0.08 \pm 0.06$ &
    $0.08 \pm 0.05$ &
    $0.10 \pm 0.09$ &
    $0.48 \pm 0.09$ &
    $\mathbf{0.03} \pm \mathbf{ 0.05}$\\
    & Prob-UNF &
    $0.08 \pm 0.04$ &
    $0.08 \pm 0.04$ &
    - & 
    - &
    $\mathbf{0.05} \pm \mathbf{0.02}$\\
    & LogProb-UNF &
    $0.23 \pm 0.19$ &
    $0.24 \pm 0.19$ & 
    - & 
    - &
    $\mathbf{0.15} \pm \mathbf{0.10}$\\ 
    \hline
    \hline
\multirow{4}{*}{Adult}
    &Accuracy& 
    $\mathbf{0.80} {\pm} \mathbf{0.01}$ &
    $\mathbf{0.80} \pm \mathbf{0.01}$ & 
    $0.78 \pm 0.02$ & 
    $0.77 \pm 0.01$ &
    $0.79 \pm 0.01$
    \\
    &Det-UNF &
    $0.08{\pm} 0.05$ &
    $\mathbf{0.06} \pm \mathbf{0.05}$ & 
    $0.08 \pm 0.08$ & 
    $0.10 \pm 0.06$ &
    $\mathbf{0.06} \pm\mathbf{ 0.04}$\\
    & Prob-UNF & 
    $0.17{\pm} 0.07$ &
    $0.12 \pm 0.07$ & 
    $-$ & 
    $-$ &
    $\mathbf{0.12} \pm \mathbf{0.07}$ \\
    & LogProb-UNF &
    $0.98{\pm} 0.55$ &
    $0.64 \pm 0.51$ & 
    $-$ & 
    $-$ &
    $\mathbf{0.56} \pm \mathbf{0.42}$\\ 
\hline
\hline
\multirow{4}{*}{Compas}&Accuracy& 
    $\textbf{0.65} {\pm} \textbf{0.01}$ &
    $\textbf{0.65} \pm \textbf{0.02}$ & 
    $0.58 \pm 0.04$ & 
    $\textbf{0.65} \pm \textbf{0.01}$ &
    $0.58 \pm 0.04$ \\
    &Det-UNF  & 
    $0.25{\pm} 0.03$ &
    $0.24 \pm 0.03$ & 
    $0.12 \pm 0.07$ & 
    $0.22 \pm 0.01$ &
    $\mathbf{0.11} \pm \mathbf{0.07}$ \\
    & Prob-UNF & 
    $0.12{\pm} 0.02$ &
    $0.11 \pm 0.02$ & 
    $-$ & 
    $-$ &
    $\textbf{0.02} \pm \textbf{0.02}$ \\
    & LogProb-UNF & 
    $0.28{\pm} 0.07$ &
    $0.24 \pm 0.07$ & 
    $-$ & 
    $-$ &
    $\textbf{0.06} \pm \textbf{0.04}$ \\
\hline 
\hline
\multirow{4}{*}{Arrhythmia}&Accuracy& 
    $\textbf{0.63} {\pm} \textbf{0.03}$ &
    $0.62 \pm 0.03$ & 
    $0.61 \pm 0.03$ & 
    $0.62 \pm 0.03$ &
    $0.61 \pm 0.03$ \\
    &Det-UNF  & 
    $0.17 {\pm} 0.08$ &
    $0.12 \pm 0.07$ & 
    $0.08 \pm 0.06$ & 
    $0.23 \pm 0.13$ &
    $\textbf{0.07} \pm \textbf{0.06}$ \\
    & Prob-UNF & 
    $0.10{\pm} 0.05$ &
    $0.06 \pm 0.04$ & 
    $-$ & 
    $-$ &
    $\textbf{0.03} \pm \textbf{0.03}$ \\
    & LogProb-UNF & 
    $0.21{\pm} 0.10$ &
    $0.14 \pm 0.08$ & 
    $-$ & 
    $-$ &
    $\textbf{0.07} \pm \textbf{0.05}$ \\ 
\hline 
\end{tabular}
\caption{Testing accuracy and unfairness (average $\pm$ standard deviation) for $N= 150$.}
\label{tab:results}
\end{table*}
\end{center}

        \noindent\textbf{Experiments with Real Data.} We now benchmark the LR, FLR and DR-FLR classifiers against fair classifiers proposed in~\cite{ref:donini2018empirical} (DOB${}^+$) and~\cite{ref:zafar2017fairness} (ZVRG) on four publicly available datasets (Adult, Drug, COMPAS, Arrhythmia\footnote{ We only use the first 12 out of 278 non-sensitive features of the Arrhythmia dataset so that we can use the same search grid for $\rho$ across all datasets (in the other datasets $p$ ranges from 5 to 12).}).
        While the Adult dataset comes with designated training and testing samples, in all other datasets we randomly select $2/3$ of the samples for training. Ultimately, the ratio of training samples to features is of the order of 10 in all datasets.
        
        To train the DR-FLR classifier, we draw 150 training samples and keep the others as validation samples. We then set $\eta = \min\{\hat p_{11}, \hat p_{01}\} / 2$, $\kappa_\mc{A}=\kappa_{\mc Y} = 0.5$ and tune $\rho \in [10^{-5}, 10^{-1}]$\footnote{After we obtain the logarithmic scale, we multiply the values by 5, and thus $\rho \in [5. 10^{-5}, 5.10^{-1}]$ at the end.} on a logarithmic search grid with 50 discretization points using the validation procedure from \cite{ref:donini2018empirical}. Using these hyperparameters, we then re-train the DR-FLR classifier on another set of 150 randomly drawn training samples. The DOB${}^+$ and ZVRG classifiers are computed using the authors' code. The accuracy and unfairness measures of all classifiers is then evaluated on the testing data.
       
        Table~\ref{tab:results} suggests that the DR-FLR classifier performs favorably relative to its competitors in that it always decreases LogProb-UNF substantially and often yields the lowest Det-UNF with only a moderate loss in accuracy.

        \noindent\textbf{Worst-Case Distribution.} Next, we visualize the extremal distribution $\QQ\opt$ from Proposition~\ref{prop:extreme-infty} for 4 pre-trained classifiers (classical logistic regression, support vector machine with RBF kernel, Gaussian processes with RBF kernel, AdaBoost). Figure~\ref{fig:quantification} illustrates which test samples are projected to the decision boundary under the adversarial distribution $\QQ\opt$ until the transportation budget corresponding to the Wasserstein radius $\rho$ is exhausted.

   \begin{figure}[h]
       \begin{subfigure}{.245\textwidth}
       \includegraphics[width=\linewidth]{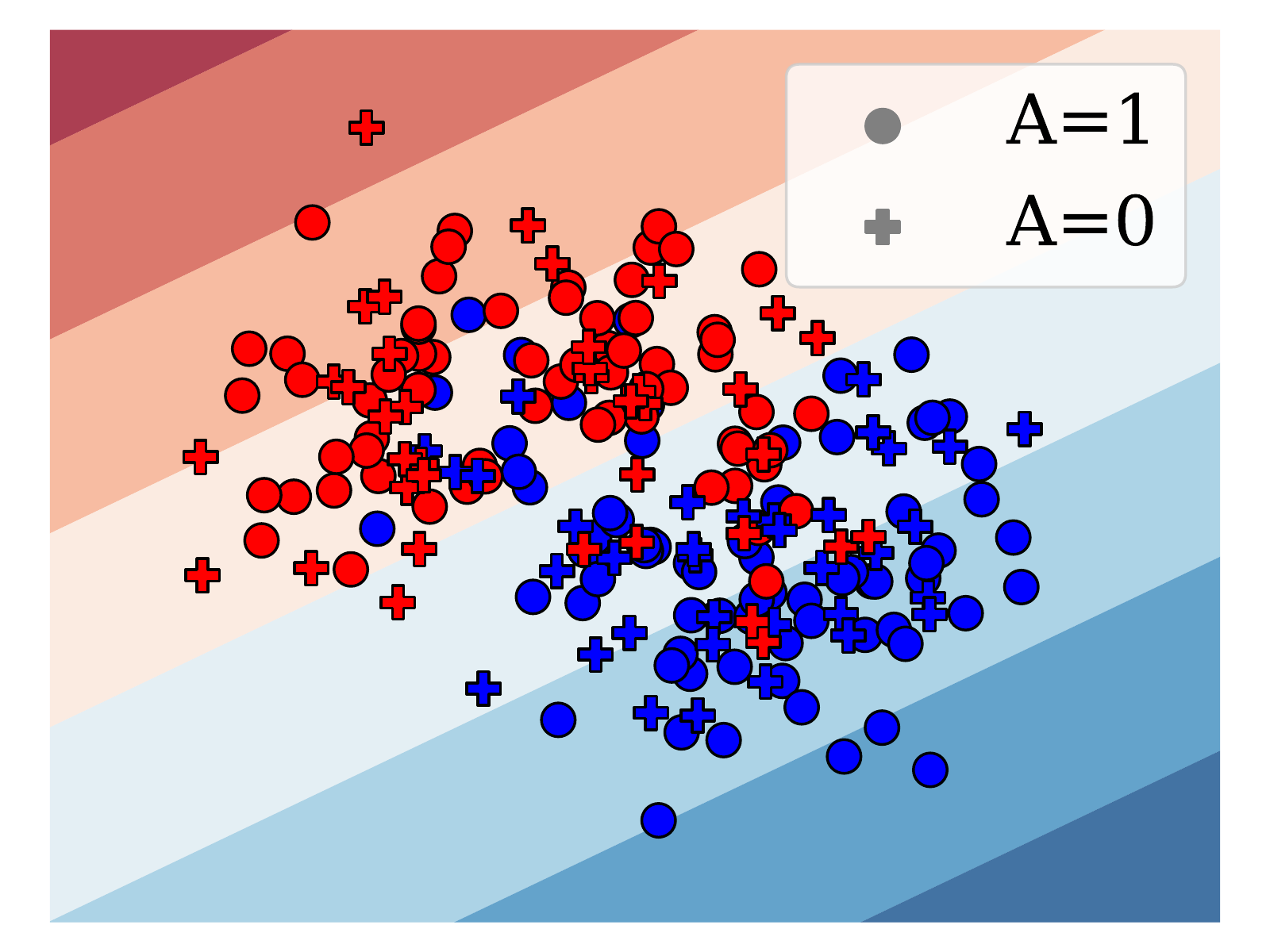}
       \end{subfigure}
       \begin{subfigure}{.245\textwidth}
       \includegraphics[width=\linewidth]{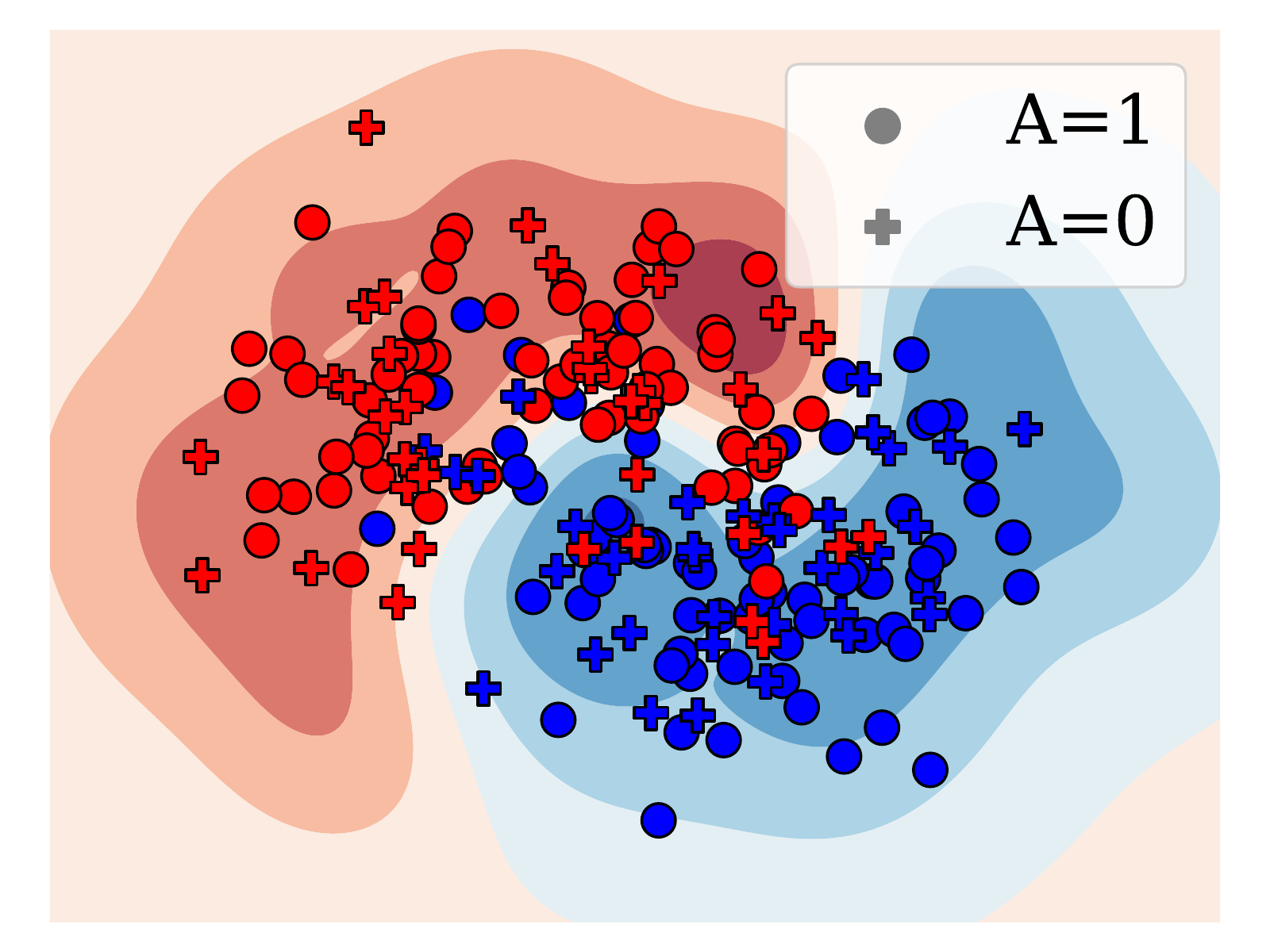}
       \end{subfigure}
       \begin{subfigure}{.245\textwidth}
       \includegraphics[width=\linewidth]{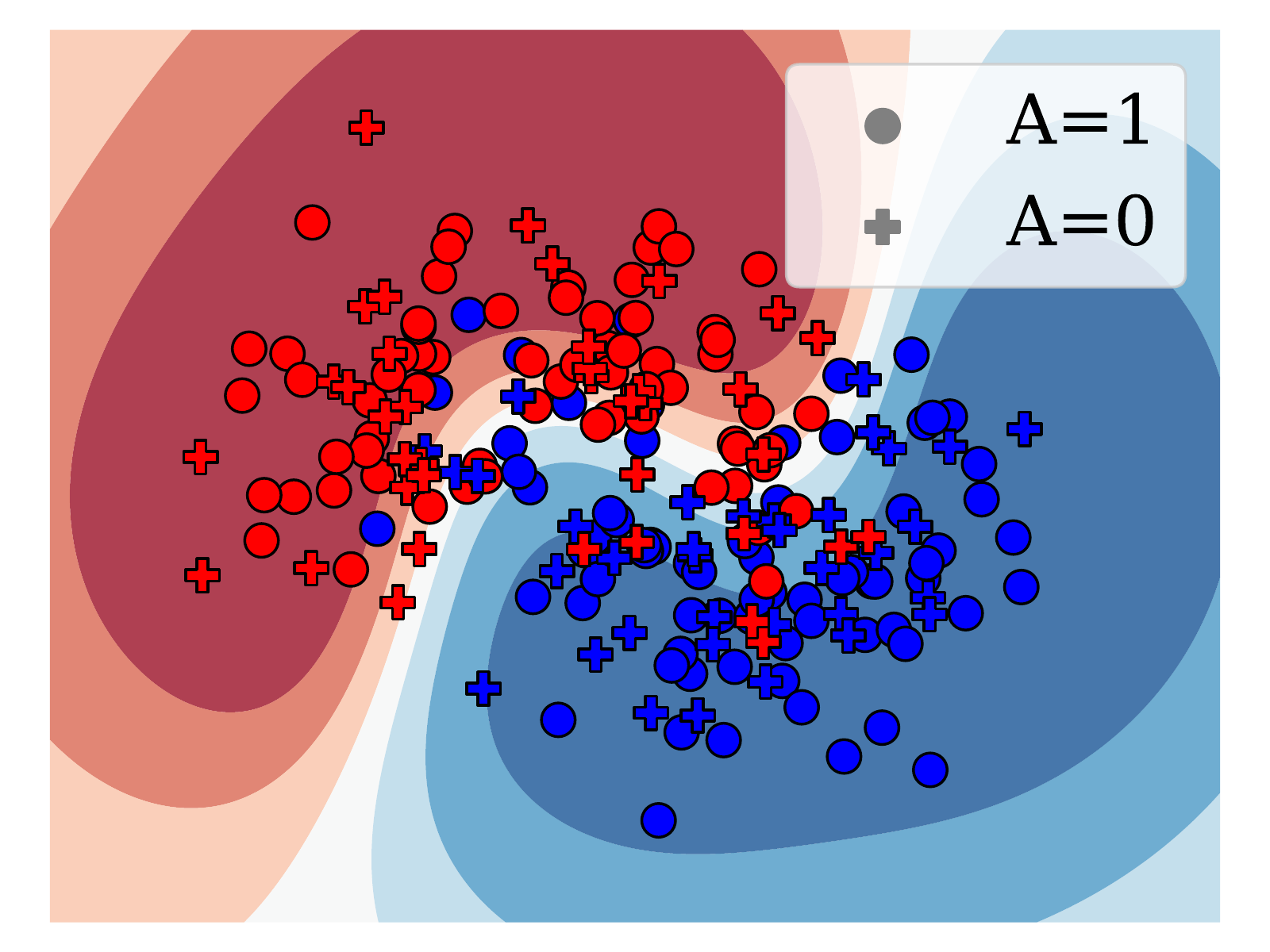}
       \end{subfigure}
       \begin{subfigure}{.245\textwidth}
       \includegraphics[width=\linewidth]{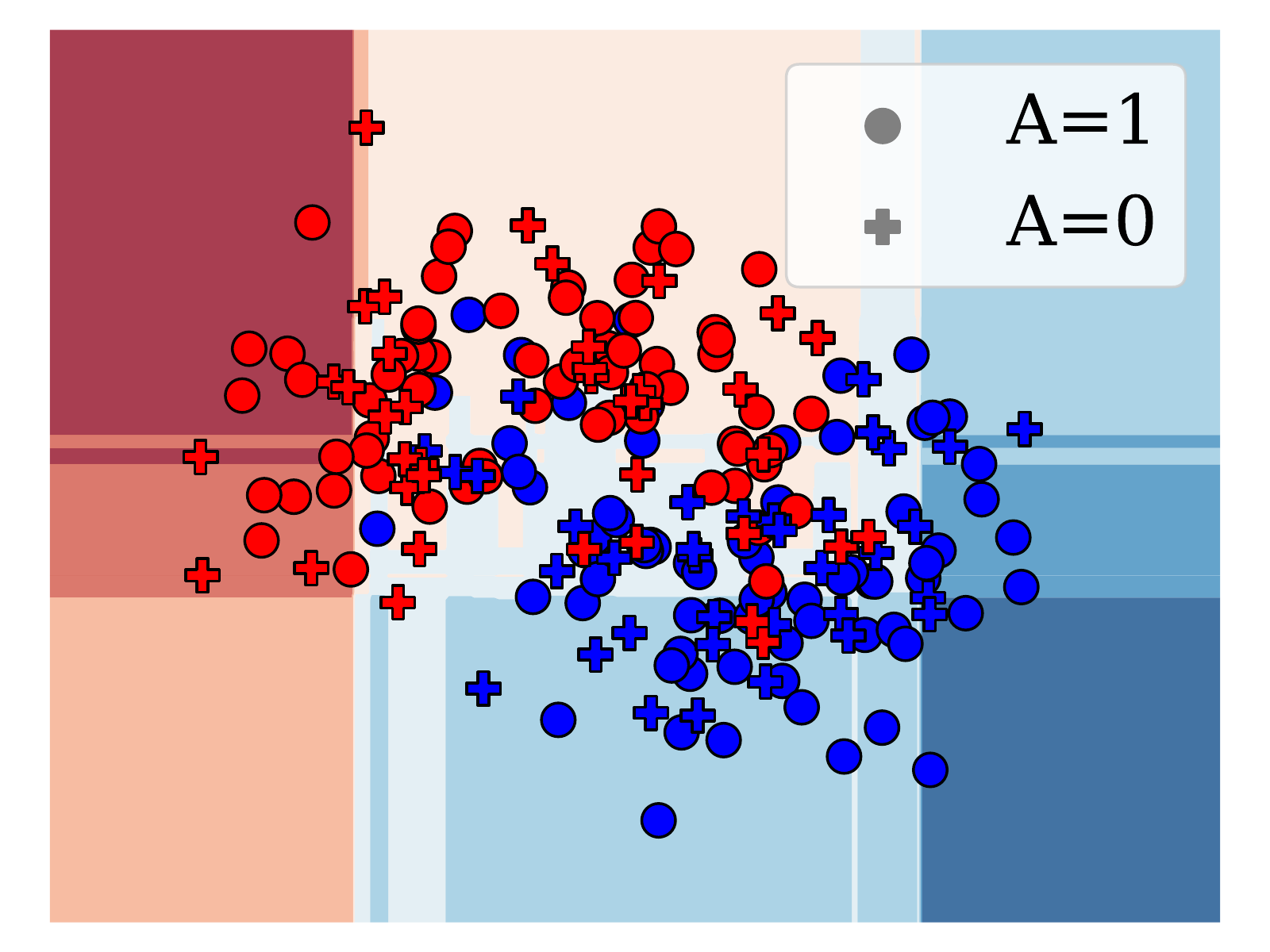}
       \end{subfigure}
       \begin{subfigure}{.245\textwidth}
       \includegraphics[width=\linewidth]{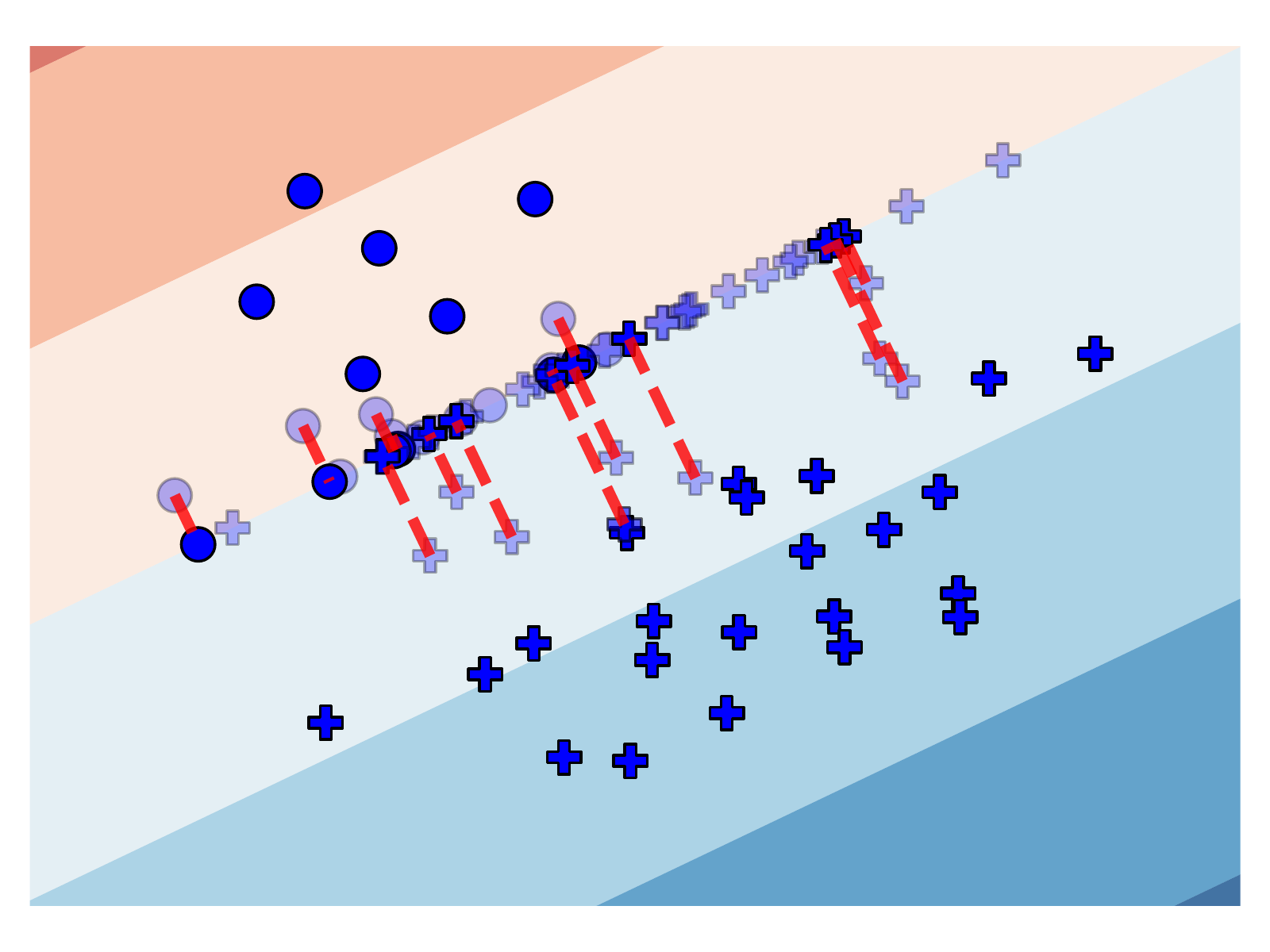}
       \caption{Logistic Regression}
       \end{subfigure}      \begin{subfigure}{.245\textwidth}
       \includegraphics[width=\linewidth]{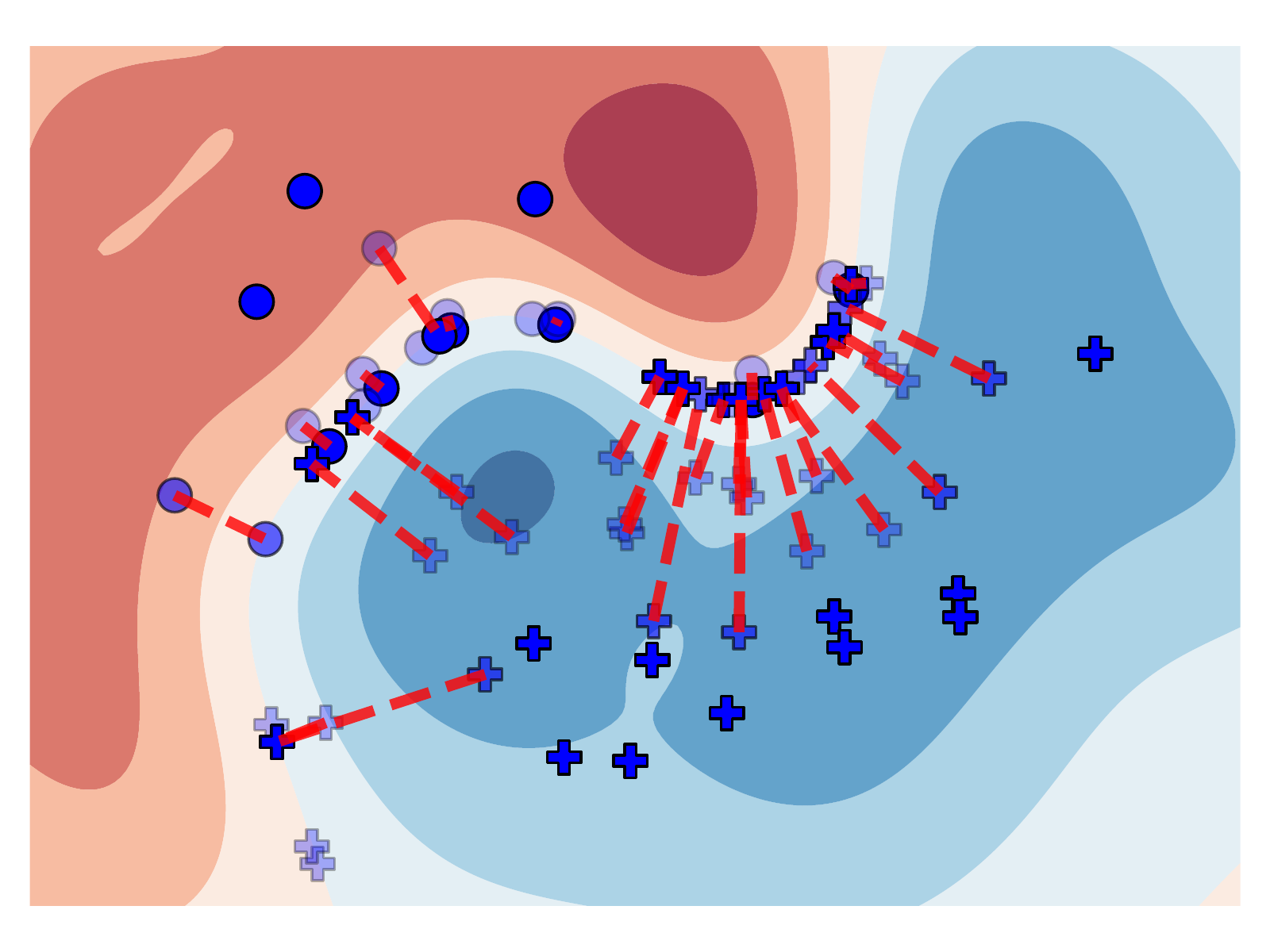}
       \caption{SVM (RBF)}
       \end{subfigure}
       \begin{subfigure}{.245\textwidth}
       \includegraphics[width=\linewidth]{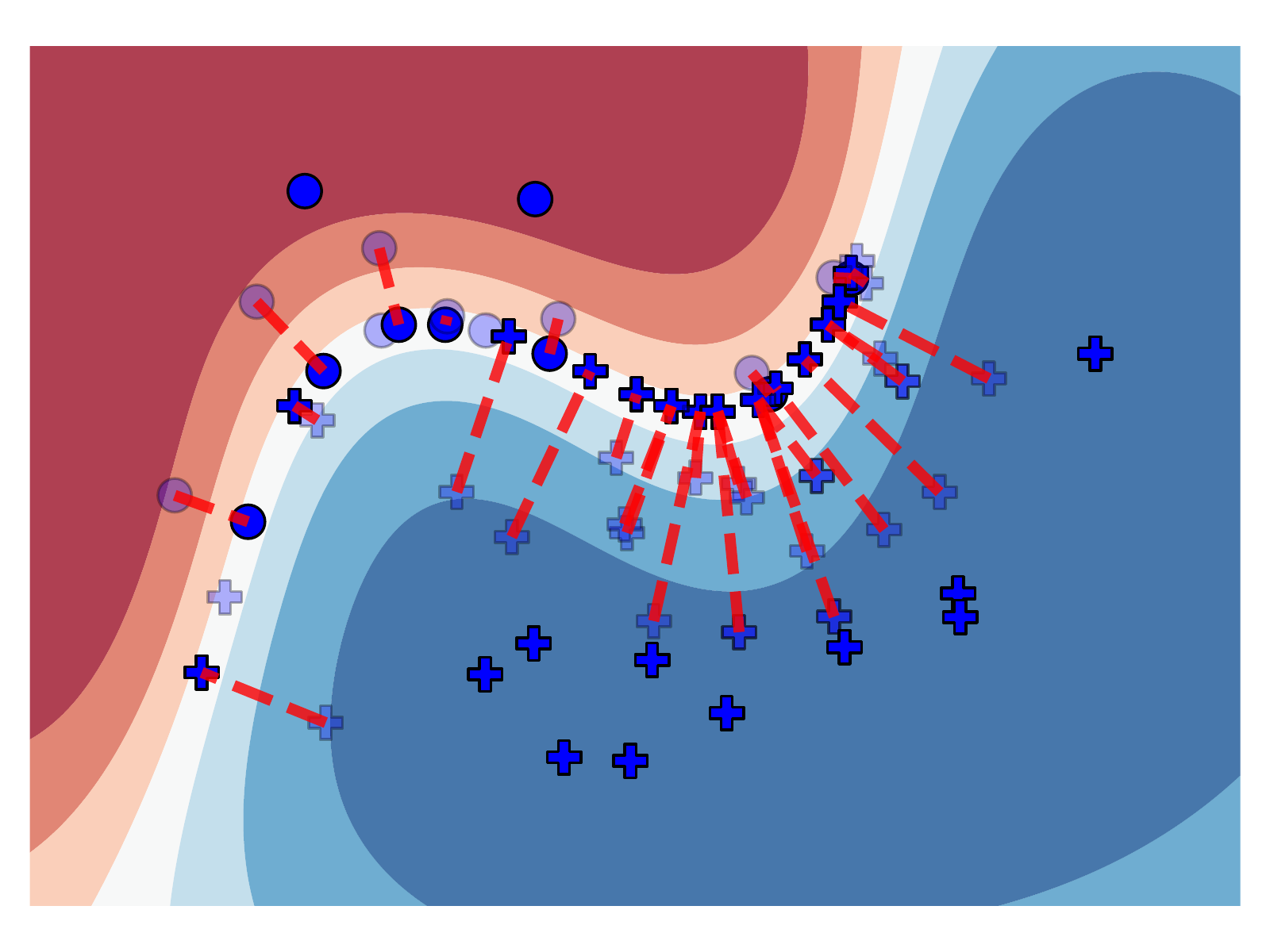}
       \caption{GP (RBF)}
       \end{subfigure}
       \begin{subfigure}{.245\textwidth}
       \includegraphics[width=\linewidth]{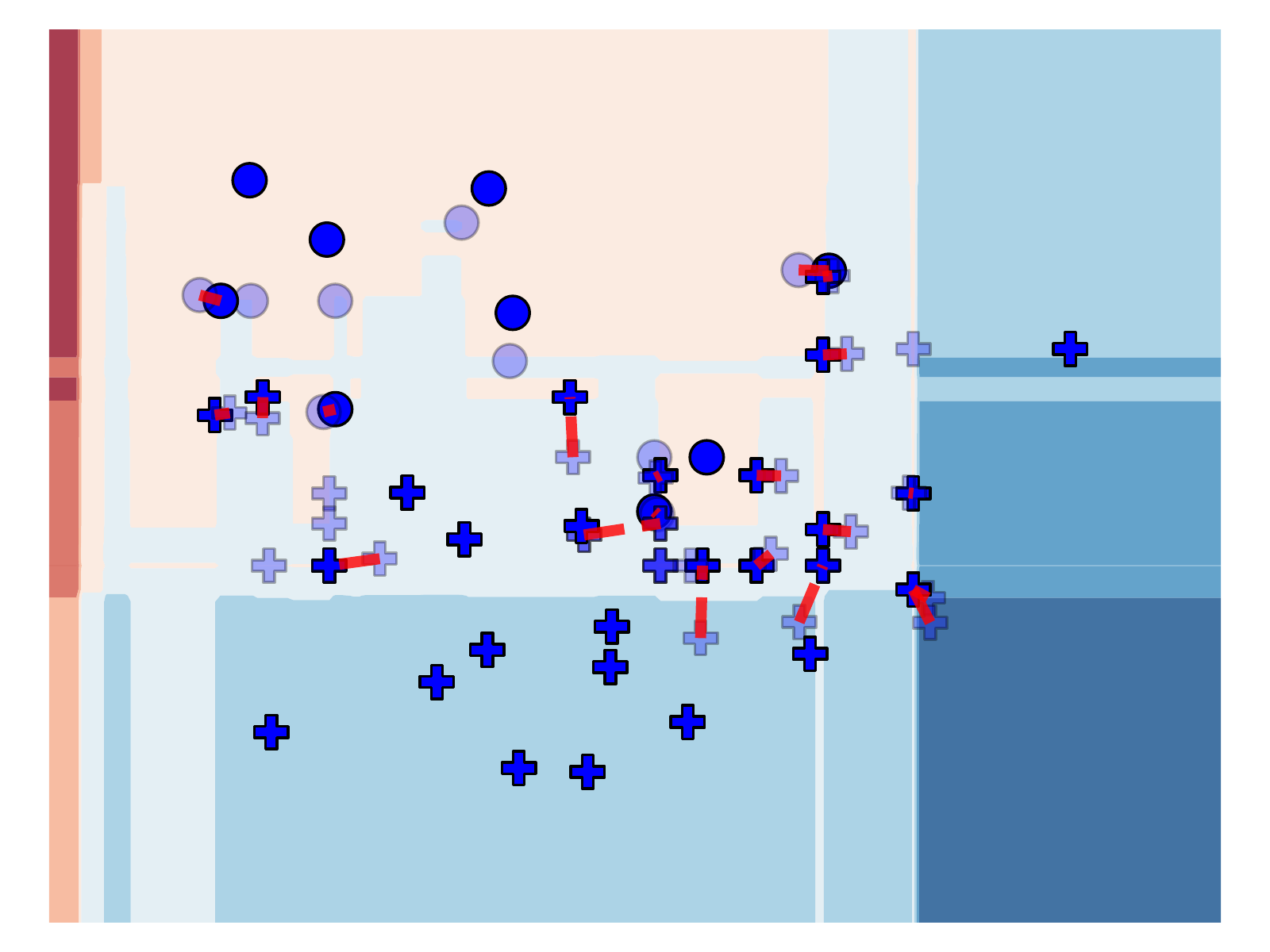}
       \caption{AdaBoost}
       \end{subfigure}
       \caption{Visualization of the extremal distribution $\QQ\opt$ for different classifiers. The red/blue background color represents the class partitions. The top row shows the test data, and the bottom row (zoomed) shows how samples with $z_i\opt > 0$ are moved to the decision boundary.}
       \label{fig:quantification}
       \vspace{-.5cm}
   \end{figure}
   \newpage{}
\section*{Appendix}
	\renewcommand\thesection{\Alph{section}}
	\renewcommand{\theequation}{A.\arabic{equation}}
	\renewcommand{\thefigure}{A.\arabic{figure}}
	\renewcommand{\thetable}{A.\arabic{table}}

	\appendix
	This appendix is organized as follows. Section~\ref{app:technical_proofs} contains all proofs omitted from the main text, while Section~\ref{sect:further_disc_numerical} provides detailed information on the numerical experiments and reports on additional numerical experiments.
    \section{Proofs}
    \label{app:technical_proofs}
	We first describe a strong semi-infinite duality result that forms the basis for several proofs. To this end, assume that $\phi: \mc X \times \mc A \times \mc Y \to \R$ is a Borel measurable loss function, and recall that $\wh p_{ a  y} = \Pnom_N(A =  a, Y =  y)$ for all $a \in \mc A$ and $y \in \mc Y$. The semi-infinite program 
	\be \label{eq:2refor-1}
        \begin{array}{cll}
        \displaystyle \sup_{\QQ\in\mathcal M} &\EE_{\QQ}[\phi(X, A, Y)]  \\ [1ex]
        \st & \Wass(\QQ, \Pnom_N) \leq \rho \\ [1ex]
        & \QQ(A =  a, Y =  y) = \wh p_{ a  y} \quad \forall a \in \mc A,\;\forall y\in \mc Y
        \end{array}
    \ee
	thus evaluates the worst-case expected loss over all distributions in a Wasserstein ball of radius $\rho\ge 0$ around the discrete nominal distribution $\Pnom_N$ under which the marginal distributions of $A$ and $Y$ coincide with their {\em nominal} marginal distributions. The following proposition generalizes existing strong duality results without marginal distribution information~\cite{ref:blanchet2019quantifying,ref:gao2016distributionally,ref:esfahani2018data,ref:zhao2018data} and can be seen as a variant of~\cite[Theorem~2]{ref:frogner2019incorporating}, which includes information on the marginal distribution of features and outputs. The proposition can also be derived from a general theory of moment problems \cite[Section~3]{ref:shapiro2001on}. We omit the proof for brevity.
	
    \begin{proposition}[Strong duality] \label{prop:conic-duality}
    If $\wh p_{ a  y}\in (0, 1)$ for all $a \in \mc A$ and $y \in \mc Y$ and if $\rho >0$, then~\eqref{eq:2refor-1} admits the strong semi-infinite dual
    \be
        \label{eq:dual}
        \begin{array}{cll}
        \inf & \rho \lambda + \ds \sum_{a \in \mc A}\sum_{y \in \mc Y} \wh p_{ a  y} \mu_{ a  y} + \frac{1}{N} \sum_{i =1}^N \nu_i \\[3ex]
        \st & \lambda \in \R_+,~\mu \in \R^{2\times 2},~\nu \in \R^N \\[1ex]
        & \lambda \, c\big((x_i, a_i, y_i), (\wh x_i, \wh a_i, \wh y_i)\big)+ \displaystyle \mu_{ a_i y_i} + \nu_i \ge  \phi(x_i, a_i, y_i) \\[1ex]
        &\hspace{4cm} \forall (x_i, a_i, y_i) \in \mc X \times \mc A \times \mc Y,\;\forall i \in [N].
        \end{array}
    \ee
    Also, if the supremum of~\eqref{eq:2refor-1} is finite, then the infimum of~\eqref{eq:dual} is attained.
    \end{proposition}

     \begin{corollary}[Absolute trust in $A$ and $Y$]
\label{corol:conic-duality-inf-kappa}
    If $\wh p_{ a  y}\in (0, 1)$ for all $a \in \mc A$ and $y \in \mc Y$ and if $\rho >0$ and $\kappa_{\mc A} = \kappa_{\mc Y}= \infty$, then~\eqref{eq:2refor-1} admits the strong semi-infinite dual
    \be
        \label{eq:dual-inf}
        \begin{array}{cll}
        \inf & \rho \lambda + \frac{1}{N} \sum_{i =1}^N \nu_i \\[1ex]
        \st & \lambda \in \R_+,~\nu \in \R^N \\[1ex]
        & \lambda \|x_i - \hat{x}_i\|+  \nu_i \ge  \phi(x_i,\hat  a_i, \hat y_i) \quad \forall x_i \in \mc X ,\;\forall i \in [N].
        \end{array}
    \ee
\end{corollary}
\begin{proof}[Proof of Corollary~\ref{corol:conic-duality-inf-kappa}]
When $\kappa_\mc{A} = \kappa_{\mc Y} = \infty$, the left hand side of the $i$-th semi-infinite constraint in~\eqref{eq:dual} evaluates to $\infty$ unless $a_i = \hat a_i$ and $y_i = \hat y_i$. In this case, the constraint is trivially satisfied and can be omitted. 
Furthermore, by definition of $\hat p_{a y}$ we have
\[\sum\limits_{a \in \mc A} \sum\limits_{y \in \mc Y} \hat p_{a y} \mu_{ a  y} = \frac{1}{N}\sum\limits_{i=1}^N \mu_{ \hat a_i  \hat y_i}.\]
Consequently, problem in~\eqref{eq:dual} reduces to
\be
        \label{eq:dual_inf}
        \begin{array}{cll}
        \inf & \rho \lambda + \ds\frac{1}{N}\ds \sum\limits_{i=1}^N \mu_{ \hat a_i  \hat y_i} + \frac{1}{N} \sum_{i =1}^N \nu_i \\[3ex]
        \st & \lambda \in \R_+,~\mu \in \R^{2\times 2},~\nu \in \R^N \\[1ex]
        & \lambda \,\|x_i -\hat x_i\|+ \displaystyle \mu_{\hat  a_i \hat y_i} + \nu_i \ge  \phi(x_i,\hat  a_i, \hat y_i) \quad \forall x_i \in \mc X,\;\forall i \in [N].
        \end{array}
    \ee
We can further simplify problem~\eqref{eq:dual_inf} by applying the change of variables $\nu_i \xleftarrow{} \mu_{\hat a_i \hat y_i} +\nu_i$, $i\in[N]$, which yields the  reformulation~\eqref{eq:dual-inf}.
This observation completes the proof.
\end{proof}
\subsection{Proofs of Section~\ref{sect:fair}}
\label{app:proof_fair_sect}
\begin{proof}[Proof of Theorem~\ref{thm:FLR}]
We define the log-loss function through
\[
    \ell_\beta(x, y) = - y \log (h_\beta(x)) - (1-y) \log (1 - h_\beta(x))\quad\forall x\in\mc X,\;\forall y\in\mc Y.
\]
By introducing an auxiliary epigraphical variable, problem~\eqref{eq:LR-true-all-f} can then be reformulated as
\begin{align}
\label{eq:epig_refor_fair}
    \begin{array}{cl}
     \Min{\beta, t} & t\\
     \st &\EE_{\PP}[\ell_\beta(X, Y)] + \eta \mathds U_f(\PP, h_\beta) \leq t.
    \end{array}
\end{align}
As $f(z) = \log(z)$ and $\PP = \Pnom_N$ by assumption, the unfairness measure simplifies to
\[
\mathds U_{f}(\Pnom_N, h_\beta) = |\EE_{\Pnom_N} [\log(h_\beta(X))| A=1, Y=1] - \EE_{\Pnom_N}[\log(h_\beta(X)) | A=0, Y=1]| . \]
By the definition of conditional expectations, we further have
\begin{align*}
    \EE_{\Pnom_N}[ \log h_\beta(X) | A = {a}, Y = 1] &= \frac{\EE_{\Pnom_N}[ \log h_\beta(X) \mathbbm{1}_{\{( a, 1)\}} (A, Y)]}{\Pnom_N(A=  a, Y = 1)} \\ & =  r_{ a}\, \EE_{\Pnom_N}[ \log h_\beta(X) \mathbbm{1}_{\{( a, 1)\}} (A, Y)]
\end{align*}
for all $a \in \mc A$, where the second equality follows from the definition of~$r_{ a}$. For any fixed $a,a'\in\mc A$ with $a\neq a'$ and $\beta\in\R^p$ we then introduce the function
\[
\hat{\mathds T}_\beta^{aa'} = \EE_{\Pnom_N}[ \ell_\beta(X, Y) + \eta r_a \log(h_\beta(X)) \mathbbm{1}_{\{(a, 1)\}}(A, Y) - \eta r_{a'} \log(h_\beta(X)) \mathbbm{1}_{\{(a', 1)\}}(A, Y)].
\]
By expanding the absolute value in the definition of $\mathds U_{f}(\Pnom_N, h_\beta)$, problem~\eqref{eq:epig_refor_fair} simplifies~to
\begin{align}
    \begin{array}{cl}
     \ds \min_{\beta,t} & t\\
     \st & \hat{\mathds T}_\beta^{10} \leq t, \; \hat{\mathds T}_\beta^{01} \leq t,
    \end{array}
    \label{eq:epigraph-for-emp}
\end{align}
which is manifestly equivalent to the optimization problem in the theorem statement. Note that by the definition of the log-loss function, we obtain \begin{align*}
    &\hat{\mathds T}_\beta^{aa'} = \EE_{\Pnom_N}[ -Y\log(h_\beta(X))-(1-Y) \log(1-h_\beta(X))+
    \eta r_a \log(h_\beta(X)) \mathbbm{1}_{\{(a, 1)\}}(A, Y) \\
    &\hspace{1cm}- \eta r_{a'} \log(h_\beta(X)) \mathbbm{1}_{\{(a', 1)\}}(A, Y)] \\
    &\; = -\frac{1}{N}\!\Big(\!\sum\limits_{\substack{i \in [N]:\\ \hat y_{i}=1 \\ \hat a_i = a}} \!(1 \!-\!\eta r_a)\log(h_\beta(\hat x_i)) \!+\!\sum\limits_{\substack{i \in [N]:\\\hat y_{i}=1\\ \hat a_i = a'\!}}\! (\eta r_{a'}\! +\!  1)\log(h_\beta(\hat x_i)) \!+\!\sum\limits_{\substack{i \in[N]:\\ \hat y_i = 0}} \!\log(1- h_\beta(\hat x_i))\Big),
\end{align*}
where the second equality holds because the expectation under the empirical distribution~$\Pnom_N$ can be expressed as a finite sum, and terms can be grouped by the labels and the sensitive attributes of the training samples. 
Thus, $\hat{ \mathds T}_\beta^{aa'}$ is convex in $\beta$ for $\eta \leq \min\{\hat p_{11}, \hat p_{01}\}$, in which case problem~\eqref{eq:epigraph-for-emp} becomes a tractable convex program.
This concludes the proof.
\end{proof}

\subsection{Proofs of Section~\ref{sect:training}}
\label{app:proofs_training}
The proof of Theorem~\ref{thm:training_refor} relies on the following simple corollary of~\cite[Lemma~1]{ref:abadeh2015distributionally}.
\begin{lemma} 
\label{lemma:conjugacy_reform} 
If $\beta \in \R^p$ and $\gamma \in \R_+$, while $g_\beta(x) = \gamma\log(1+ \exp(-\langle \beta , x \rangle))$ is a convex function of $x\in \R^p$, then we have 
\[\Sup{x \in \R^p}~\gamma g_\beta(x)-\lambda \|x - \hat{x}\| = \begin{cases}
\gamma g_\beta(\hat x)   \quad &\text{if}~\gamma\|\beta\|_* \leq {\lambda}\\
+\infty &\text{otherwise}
\end{cases}\]
for all $\lambda \in \R_{++}$, where $\|\cdot \|_{*}$ represents the dual norm of $\|\cdot \|$.
\end{lemma}
\begin{proof}[Proof of Theorem~\ref{thm:training_refor}]
To simplify notation, we define the log-loss function as usual as
\[
    \ell_\beta(x, y) = - y \log h_\beta(x) - (1-y) \log (1 - h_\beta(x))\quad \forall x\in\mc X,~\forall y\in\mc Y.
\]
By introducing an auxiliary epigraphical variable, problem~\eqref{eq:dro_fair_training} can then be reformulated as
\begin{align}
\label{eq:epig_refor_fair_dro}
    \begin{array}{cl}
     \Min{\beta, t} & t\\
     \st &\sup\limits_{\QQ \in \mathbb{B}_\rho(\Pnom_N)}\EE_{\QQ}[\ell_\beta(X, Y)] + \eta \mathds U_f(\QQ, h_\beta) \leq t.
    \end{array}
\end{align}
As $f(z) = \log(z)$ by assumption, the unfairness measure simplifies to
\[
\mathds U_{f}(\QQ, h_\beta) = |\EE_{\QQ} [\log(h_\beta(X))| A=1, Y=1] - \EE_{\QQ}[\log(h_\beta(X)) | A=0, Y=1]| . \]

By the definition of conditional expectations, we have for all $\QQ \in \mbb B_\rho(\Pnom_N)$ and $a \in \mc A$ that
\begin{align*}
    \EE_\QQ[ \log h_\beta(X) | A = {a}, Y = 1] &= \frac{\EE_\QQ[ \log h_\beta(X) \mathbbm{1}_{\{( a, 1)\}} (A, Y)]}{\QQ(A=  a, Y = 1)}\\
    &= r_{ a} \EE_\QQ[ \log h_\beta(X) \mathbbm{1}_{\{( a, 1)\}} (A, Y)],
\end{align*}
where the second equality holds because $\QQ(A=  a, Y = 1) = \wh p_{ay} = 1/r_{ a}$ for any $\QQ \in \mbb B_\rho(\Pnom_N)$. 
For any fixed $a,a'\in\mc A$ with $a\neq a'$ and $\beta\in\R^p$ we then introduce the function
\[
        \phi_\beta^{aa'}(\tilde x,\tilde a,\tilde y)= \ell_\beta (\tilde x,\tilde y)  +\eta ~ r_a \log(h_\beta(\tilde x))\mathbbm{1}_{\{(a,1)\}} (\tilde a, \tilde y) - \eta ~r_{a'}\log(h_\beta(\tilde  x)) \mathbbm{1}_{\{(a',1)\}} (\tilde a, \tilde y)
\]
    of $\tilde x\in\mc X$, $\tilde a\in\mc A$ and $\tilde y\in\mc Y$, and we define 
\[
\mathds T_\beta^{aa'} = \Sup{\QQ \in \mbb B_\rho(\hat \PP_N)} \EE_\QQ[\phi_\beta^{aa'}(X,A,Y) ].
\]
The integrand $\phi_\beta^{aa'}$ satisfies the linear growth condition of \cite[Theorem~2.2]{ref:yue2020linear}, which guarantees that $\mathds T_\beta^{aa'}$ is finite. By using the above notational conventions and introducing an auxiliary epigraphical variable as in the proof of Theorem~\ref{thm:FLR}, problem~\eqref{eq:epig_refor_fair_dro} is simplified to
\begin{align}
    \begin{array}{cl}
     \ds \min_{\beta,t} & t\\
     \st & \mathds T_\beta^{10} \leq t, \; \mathds T_\beta^{01} \leq t.
    \end{array}
    \label{eq:epigraph-for}
\end{align}
To convert problem~\eqref{eq:dro_fair_training} to a convex program, we need to simplify the constraints that involve~$\mathds T_\beta^{aa'}$. To this end, we may use Proposition~\ref{prop:conic-duality} to obtain
    \be
        \label{eq:prob_dual-1}
        \mathds T_\beta^{aa'}=\left\{
        \begin{array}{cll}
        \min & \rho \lambda + \ds \sum_{a \in \mc A }\sum_{y \in \mc Y} \hat p_{a y} \mu_{ a y} + \frac{1}{N} \sum_{i =1}^N \nu_i \\[3ex]
        \st & \lambda \in \R_+,~\mu \in \R^{2\times 2},~\nu \in \R^N \\[1ex]
        & \lambda \, c\big((x_i, a_i, y_i), (\wh x_i, \wh a_i, \wh y_i)\big)+ \mu_{a_i y_i} + \nu_i \ge  \phi_\beta^{aa'}(x_i, a_i, y_i) \\[1ex]
        &\hspace{3cm} \forall (x_i, a_i, y_i) \in \mc X \times \mc A \times \mc Y,\;\forall i \in [N].
        \end{array}\right.
    \ee
    As $\mathds T_\beta^{aa'}$ is finite, Proposition~\ref{prop:conic-duality} also ensures that the minimum of problem~\eqref{eq:prob_dual-1} is attained.
    
    We now investigate the $i$-th semi-infinite constraint in~\eqref{eq:prob_dual-1} for a fixed $a_i$ and $y_i$. Thanks to the additive separability of the transportation cost, this constraint can be reformulated~as
    \be \label{eq:dual-constraint}
        \nu_i \ge \Sup{x_i \in \mc X} \left\{ \phi_\beta^{aa'}(x_i, a_i, y_i) - \lambda \| x_i - \wh x_i \| \right\} - \kappa_{\mc A} | a_i - \wh a_i| \lambda - \kappa_{\mc Y} | y_i - \wh y_i| \lambda - \mu_{a_i y_i} 
    \ee
    If $y_i = 0$ and $a_i\in\mc A$, then $\phi_\beta^{aa'}(x_i, a_i, 0) = -\log(1-h_\beta(x_i))$, and by Lemma~\ref{lemma:conjugacy_reform}, we have
    \begin{align*}
        \Sup{x_i \in \mc X} \left\{ \phi_\beta^{aa'}(x_i, a_i, 0) - \lambda \| x_i - \wh x_i \| \right\} &= \Sup{x_i \in \x} -\log(1-h_\beta(x_i))-\lambda\| x_i - \hat x_i\| \\
        &= 
        \begin{cases}
        -\log(1-h_\beta( \hat x_i)) \quad &\text{if}~ \|\beta\|_* \leq \lambda, \\
        +\infty &\text{otherwise},
        \end{cases}
    \end{align*}
    which implies that the constraint~\eqref{eq:dual-constraint} is equivalent to the inequalities
    \[
        \| \beta \|_* \le \lambda \quad\text{and}\quad \nu_i \ge -\log(1-h_\beta(\hat x_i)) - \kappa_{\mc A} | a_i - \wh a_i| \lambda - \kappa_{\mc Y} | \wh y_i| \lambda - \mu_{a_i 0}.
    \]
    If $a_i = a$ and $y_i = 1$, then $\phi_\beta^{aa'}(x_i, a, 1) = (\eta r_a - 1) \log(h_\beta(x_i))$, and by Lemma~\ref{lemma:conjugacy_reform}, we have
    \[
        \Sup{x_i \in \mc X} \left\{ \phi_\beta^{aa'}(x_i, a, 1) - \lambda \| x_i - \wh x_i \| \right\} =
        \begin{cases}
        (\eta r_a - 1) \log(h_\beta(\wh x_i)) \quad &\text{if}~  (1 -\eta r_a)\|\beta\|_* \leq \lambda, \\
        +\infty &\text{otherwise},
        \end{cases}
    \]
    which implies that the constraint~\eqref{eq:dual-constraint} is equivalent to
    \[
        (1 -\eta r_a)\|\beta\|_* \le \lambda\quad\text{and} \quad \nu_i \ge (\eta r_a -1) \log(h_\beta(\wh x_i)) - \kappa_{\mc A} | a - \wh a_i| \lambda - \kappa_{\mc Y} |1 -
                \wh y_i| \lambda - \mu_{a 1}.
    \]
    If $a_i = a'$ and $y_i = 1$, finally, then $\phi_\beta^{aa'}(x_i, a', 1) = -(1+ \eta r_{a'}) \log(h_\beta(x_i))$, and we can use an analogous argument involving Lemma~\ref{lemma:conjugacy_reform} to show that
    \[
        \Sup{x_i \in \mc X} \left\{ \phi_\beta^{aa'}(x_i, a', 1) - \lambda \| x_i - \wh x_i \| \right\} =
        \begin{cases}
        -(1 +\eta~ r_{a'}) \log(h_\beta(\wh x_i)) \quad &\text{if}~  (1 +\eta~ r_{a'})\|\beta\|_* \leq \lambda, \\
        +\infty &\text{otherwise},
        \end{cases}
    \]
    which implies that the constraint~\eqref{eq:dual-constraint} is equivalent to
    \[
        (1 + \eta~ r_{a'})\|\beta\|_* \le \lambda\quad \text{and} \quad \nu_i \ge - (1 + \eta~ r_{a'}) \log(h_\beta(\wh x_i)) - \kappa_{\mc A} | a' - \wh a_i| \lambda - \kappa_{\mc Y} |1 -
        \wh y_i| \lambda - \mu_{{a'} 1}.
    \]
      Substituting the above reformulations of constraint~\eqref{eq:dual-constraint} corresponding to all possible combinations of $a_i$ and $y_i$ into~\eqref{eq:prob_dual-1} yields
     \[
        \mathds T_\beta^{aa'}=\left\{ \!\! \begin{array}{c@{\;}ll}
        \min & \rho \lambda + \ds \sum_{a \in \mc A}\sum_{y \in \mc Y} \hat p_{a y} \mu_{a y} + \frac{1}{N} \sum_{i =1}^N \nu_i \\[3ex]
        \st & \lambda \in \R_+,~\mu \in \R^{2 \times 2},~\nu \in \R^N \\
        & \|\beta\|_* \leq \lambda , ~ \| \beta \|_*(1 - \eta r_a) \leq \lambda,~
        \| \beta \|_*(1 + \eta r_{a'}) \leq \lambda\\
        & \!\!\! \left. \begin{array}{ll}
        \nu_i  \geq -\log(1 - h_\beta(\hat x_i)) - \kappa_\mc{A} |a-\hat a_i| \lambda -\kappa_{\mc Y}|\hat y_i|\lambda -\mu_{a 0} \\
        \nu_i  \geq -\log(1 - h_\beta(\hat x_i)) - \kappa_\mc{A} |a'-\hat a_i| \lambda -\kappa_{\mc Y}|\hat y_i|\lambda -\mu_{a' 0} \\
        \nu_i \geq (\eta r_a-1) \log(h_\beta(\hat x_i)) - \kappa_\mc{A} |a-\hat a_i| \lambda -\kappa_{\mc Y}|1 -\hat y_i|\lambda -\mu_{a 1} \\
        \nu_i \ge -(1+\eta r_{a'})\log(h_\beta(\hat x_i)) - \kappa_\mc{A} |a'-\hat a_i| \lambda -\kappa_{\mc Y}|1 -\hat y_i|\lambda -\mu_{a' 1} 
        \end{array} \!\!\! \right\} \forall i \in [N].
        \end{array}\right.
    \]
    Note that the constraints $\|\beta\|_* \leq \lambda$ and $\|\beta \|_* (1- \eta r_a) \leq \lambda$ are redundant in view of the constraint $\|\beta\|_* (1+\eta r_{a'}) \leq \lambda$. The claim then follows by substituting the dual reformulations for $\mathds T_\beta^{aa'}$ into~\eqref{eq:epigraph-for} and eliminating the embedded minimization operators.
\end{proof}
\subsection{Proofs of Section~\ref{sect:quantify}}
\label{app:proofs_quantifying}
\begin{proof}[Proof of Theorem~\ref{thm:quantify-EO}]
By the definition of $\VV(a,a')$ for $a,a'\in\mc A$, one readily verifies that the bounds on the unfairness measure can be expressed as
\[
    \overline{\mathds U}_f = \max\{ \VV(1, 0), \VV(0, 1)\}\quad \text{and} \quad \underline{\mathds U}_{f} = \max\{0, - \VV(1, 0), - \VV(0, 1) \}.
\]
    For any fixed $a,a'\in \mc A$ with $a \neq a'$ we then introduce the function
    \[
        \phi^{aa'}(\tilde x, \tilde a, \tilde y) =  r_a  \mathbbm{1}_{\mc X_1  \times \{( a,1)\}} (\tilde x, \tilde a, \tilde y) - r_{a'}  \mathbbm{1}_{\mc X_1\times\{( a',  1)\}} (\tilde x, \tilde a, \tilde y),
    \]
    which depens on $\tilde x\in\mc X$, $\tilde a\in\mc A$ and $\tilde y\in\mc Y$, and which allows us to re-express $\VV(a, a')$ as
    \[
        \VV(a, a') = \left\{
        \begin{array}{cl}
        \ds \sup_{\QQ\in\mc M} & ~\EE_\QQ[\phi^{aa'}(X, A, Y)]\\ [1ex]
        \st & \Wass(\QQ, \Pnom_N) \leq \rho \\ [1ex]
        & \QQ(A =  a, Y =  y) = \wh p_{ a  y} \quad \forall a \in \mc A, ~\forall y \in \mc Y.
        \end{array}
        \right.
    \]
    Note that the function $\phi^{aa'}$ is piecewise constant and thus bounded, which implies that $\VV(a, a')$ is finite. The strong duality result from Proposition~\ref{prop:conic-duality} further implies that
    \be \label{eq:dual-problem-2}
        \VV(a, a') = \left\{\begin{array}{cll}
        \min & \rho \lambda + \wh p^\top \mu + \frac{1}{N} \mathbf{1}^\top \nu \\[1ex]
        \st & \lambda \in \R_+,~\mu \in \R^4,~\nu \in \R^N \\[0.5ex]
        &  \lambda c\big((x_i, a_i, y_i), (\wh x_i, \wh a_i, \wh y_i)\big)+ \mu_{a_i y_i} + \nu_i \ge  \phi^{aa'}(x_i, a_i, y_i) \\[0.5ex]
        &\hspace{3.5cm} \forall (x_i, a_i, y_i) \in \mc X \times \mc A \times \mc Y,\;\forall i \in [N].
        \end{array}\right.
    \ee
    Note that the minimum of problem~\eqref{eq:dual-problem-2} is attained because $\VV(a, a')$ is finite. By the definition of the transportation cost, the $i$-th semi-infinite constraint in~\eqref{eq:dual-problem-2} can be expressed more explicitly as
    \be \label{eq:dual-constraint-2}
        \begin{array}{r}
        \ds \nu_i \ge \Sup{x_i \in \mc X} \left\{ \phi^{aa'}(x_i, a_i, y_i) - \lambda \| x_i - \wh x_i \| \right\} - \kappa_{\mc A} | a_i - \wh a_i| \lambda - \kappa_{\mc Y} | y_i - \wh y_i| \lambda - \mu_{a_i y_i} \phantom{a}\\[2ex] \forall a_i \in \mc A,~\forall y_i \in \mc Y.
        \end{array}
    \ee
    If $y_i = 0$ and $a_i \in \mc A$, then $\phi^{aa'}(x_i, a_i, 0) = 0$, and thus~\eqref{eq:dual-constraint-2} simplifies to
    \[
        \nu_i \ge - \kappa_{\mc A} | a_i - \wh a_i| \lambda - \kappa_{\mc Y} | \wh y_i| \lambda - \mu_{a_i 0} \quad \forall a_i \in \mc A.
    \]
    If $a_i = a$ and $y_i = 1$, then $\phi^{aa'}(x_i, a, 1) = r_a \mathbbm{1}_{\mc X_1} (x_i)$, and we have
    \begin{align*}
        \Sup{x_i \in \mc X}~r_a \mathbbm{1}_{\mc X_1}(x_i) - \lambda \| x_i - \wh x_i \| 
        &= \begin{cases}
            r_a & \text{if } \wh x_i \in \mc X_1 \\
            \max\{0, r_a - \lambda d_{1i} \} &\text{if } \wh x_i \not\in \mc X_1
        \end{cases} \\
        &= \max\{0, r_a - \lambda d_{1i}\},
    \end{align*}
    where the last equality holds because $d_{1i} = 0$ if $\wh x_i \in \mc X_1$. Thus, constraint~\eqref{eq:dual-constraint-2} reduces to
    \[
        \nu_i \ge \max\{0, r_a - \lambda d_{1i}\} - \kappa_{\mc A} | a - \wh a_i| \lambda - \kappa_{\mc Y} | 1 - \wh y_i| \lambda - \mu_{a 1}.
    \]
    If $a_i = a'$ and $y_i = 1$, finally, then we have $\phi^{aa'}(x_i, a', 1) = -r_{a'} \mathbbm{1}_{\mc X_1} (x_i)$, and thus
    \begin{align*}
        \Sup{x_i \in \mc X}~-r_{a'} \mathbbm{1}_{\mc X_1}(x_i) - \lambda \| x_i - \wh x_i \| 
        &= \begin{cases}
            \max\{-r_{a'}, - \lambda d_{0i} \} & \text{if } \wh x_i \in \mc X_1 \\
             0 &\text{if } \wh x_i \not\in \mc X_1
        \end{cases} \\
        &=  \max\{- r_{a'},  - \lambda d_{0i} \},
    \end{align*}
    where the last equality holds because $d_{0i} = 0$ whenever $\wh x_i \not\in\mc X_1$. Because
    the set $\mc X_1$ is closed, the supremum in the above expression is not attained. Constraint~\eqref{eq:dual-constraint-2} now becomes
    \[
        \nu_i \ge \max\{- r_{a'},  - \lambda d_{0i} \} - \kappa_{\mc A} | a' - \wh a_i| \lambda - \kappa_{\mc Y} | 1 - \wh y_i| \lambda - \mu_{a' 1}.
    \]
    In summary, the semi-infinite constraint~\eqref{eq:dual-constraint-2} is equivalent to the six linear constraints 
    \[
        \begin{array}{l}
            \nu_i \ge - \kappa_{\mc A} |a - \wh a_i | \lambda - \kappa_{\mc Y} | \wh y_i | \lambda - \mu_{a0} \\
            \nu_i \ge - \kappa_{\mc A} |a' - \wh a_i | \lambda - \kappa_{\mc Y} | \wh y_i | \lambda - \mu_{a'0} \\
            \nu_i \ge r_a - \lambda d_{1i} - \kappa_{\mc A} |a - \wh a_i | \lambda - \kappa_{\mc Y} |1- \wh y_i | \lambda - \mu_{a1} \\
            \nu_i \ge - \kappa_{\mc A} |a - \wh a_i | \lambda - \kappa_{\mc Y} |1- \wh y_i | \lambda - \mu_{a1} \\
            \nu_i \ge -r_{a'}  - \kappa_{\mc A} |a' - \wh a_i | \lambda - \kappa_{\mc Y} |1- \wh y_i | \lambda - \mu_{a'1}\\
            \nu_i \ge  - \lambda d_{0i} - \kappa_{\mc A} |a' - \wh a_i | \lambda - \kappa_{\mc Y} |1- \wh y_i | \lambda - \mu_{a'1}.
            \end{array}
    \]
    The claim now follows by substituting this reformulation into~\eqref{eq:dual-problem-2} for every $i\in[N]$.
\end{proof}
\begin{proof}[Proof of Theorem~\ref{thm:quantify-EO-infinite}]
        If $\kappa_{\mc A} = \kappa_{\mc Y} = \infty$, then the linear programming reformulation derived in Theorem~\ref{thm:quantify-EO} simplifies to
        \be
        \label{eq:worst_case_inf-1}
            \begin{array}{cll}
                \min & \rho \lambda +\ds \sum\limits_{a \in \mc A}\sum\limits_{y \in\mc Y} \wh p_{ay} \mu_{ay} + \ds \frac{1}{N} \sum\limits_{i=1}^N \nu_i \\
                \st & \lambda \in \R_+,~\mu \in \R^{2 \times 2},~\nu \in \R^N \\
                & \hspace{-2mm}
                \left. \begin{array}{lll}
                \nu_i + \mu_{\wh a_i 0} \ge 0 & \text{ if } \wh y_i = 0 \\
                \nu_i + \mu_{a1} + d_{1i} \lambda \ge r_a & \text{ if } \wh a_i = a, \wh y_i = 1  \\
                 \nu_i + \mu_{a1} \ge 0 & \text{ if } \wh a_i = a, \wh y_i = 1 \\
                 \nu_i + \mu_{a'1} \ge -r_{a'} & \text{ if } \wh a_i = a', \wh y_i = 1 \\
             \nu_i + \mu_{a'1} +  d_{0i} \lambda \ge 0 & \text{ if } \wh a_i = a', \wh y_i = 1 \end{array}
             \right\}
             \forall i \in [N].
            \end{array}
        \ee
        Furthermore, the first constraint $\mu_{\hat a_i} \ge -\nu_i$ force $\mu_{\hat a_i 0} = -\nu_i$ for all $\{i \in [N]: \hat y_i = 0\}$.
        Hence, by definition of $\hat p_{a y}$ we have
        \[\sum\limits_{a \in \mc A}\hat p_{a0}\mu_{ a  0} = -\frac{1}{N}\sum\limits_{i \in [N]: \hat y_i = 0} \nu_i.\]
        Consequently, by defining the sets $\bar{\mc I}_a = \{ i \in [N]: \wh a_i = a, \wh y_i = 1\}$ and $\bar{\mc I}_{a'} = \{ i \in [N]: \wh a_i = a', \wh y_i = 1\}$, problem in~\eqref{eq:worst_case_inf-1} is further simplified to
        \[
            \begin{array}{cll}
                \min & \rho \lambda + \wh p_{a1} \mu_{a1} + \wh p_{a'1} \mu_{a'1} + \frac{1}{N} \sum\limits_{i \in \bar{\mc  I}_{a} \cup \bar{\mc  I}_{a'}} \nu_i \\
                \st & \lambda \in \R_+,~\mu \in \R^{2 \times 2},~\nu \in \R^N \\
                & \hspace{-2mm} \left.
                \begin{array}{l}
                \nu_i + \mu_{a1} + d_{1i} \lambda \ge r_a \\
                \nu_i + \mu_{a1} \ge 0  \\
                \end{array} 
                \right\}  \forall i \in \bar{\mc I}_{a}\\[2.5ex]
                & \hspace{-2mm} \left.
                \begin{array}{l}
                \nu_i + \mu_{a'1} \ge -r_{a'} \\
                \nu_i + \mu_{a'1} +  d_{0i} \lambda \ge 0 
                \end{array} \right\}  \forall i \in \bar{\mc I}_{a'}.
            \end{array}
        \]
        By introducing the Lagrangian multipliers $\gamma_{1}, \gamma_2 \in \R_+^{|\bar{\mc I}_a|}$ and $\gamma_3, \gamma_4 \in \R_+^{|\bar{\mc I}_{a'}|}$, we obtain the linear dual problem of the above problem as 
        \be
        \label{eq:quan_kappa_inf_dual}
            \begin{array}{cll}
                \max & r_a \sum\limits_{i \in \bar{\mc I}_a} \gamma_{1i} - r_{a'} \sum\limits_{i \in \bar{\mc I}_{a'}} \gamma_{3i}  \\
                \st & \gamma_{1} \in \R_+^{|\bar{\mc I}_a|},\; \gamma_2 \in \R_+^{|\bar{\mc I}_a|},\; \gamma_3 \in \R_+^{|\bar{\mc I}_{a'}|},\; \gamma_4 \in \R_+^{|\bar{\mc I}_{a'}|}
                \\[1.5ex]
                & \rho - \sum\limits_{i \in  \bar{\mc I}_a} \gamma_{1i} d_{1i} - \sum\limits_{i \in \bar{\mc I}_{a'}} \gamma_{4i} d_{0i} \geq 0 \\
                & \hat p_{a1} - \sum\limits_{i \in \bar{\mc I}_a} \left(\gamma_{2i} + \gamma_{1i} \right)= 0  \\
                & \hat p_{a'1} - \sum\limits_{i \in \bar{\mc I}_{a'}} \left(\gamma_{3i} + \gamma_{4i}\right) = 0  \\
                & {1}/{N} - \gamma_{1i} - \gamma_{2 i} = 0 & \forall i \in \bar{\mc I}_a  \\
                & {1}/{N} - \gamma_{3i} - \gamma_{4i} = 0 & \forall i \in \bar{\mc I}_{a'}.
            \end{array}
        \ee
        We now define the sets
        $\mc I_{a} = \{i\in\bar{\mc I}_a : \hat x_i \in \text{int}(\mc X_0)\}$ and $\mc I_{a'} = \{i\in \bar{\mc I}_{a'} : \hat x_i \in \text{int}(\mc X_1)\}$.
        Due to the last two constraints, $\gamma_{2i} +\gamma_{1i} = 1/N$ and $\gamma_{3i} +\gamma_{4i} = 1/N$, the third and the forth constraints of~\eqref{eq:quan_kappa_inf_dual} become redundant as $\sum_{i \in \bar{\mc I}_a} 1/N = \hat p_{a1}$ and $\sum_{i \in \bar{\mc I}_{a'}} 1/N= \hat p_{a'1}$ by definition of the sets $\bar{\mc I}_a$ and $\bar{\mc I}_{a'}$.
        Notice that due to last constraint in \eqref{eq:quan_kappa_inf_dual}, we have $\gamma_{3i} = 1/N - \gamma_{4i}$ for all $i\in \bar{\mc I}_{a'}$.
        Then, we can further simplify problem
        \eqref{eq:quan_kappa_inf_dual} to
              \be
        \label{eq:quan_kappa_inf_dual-2}
            \begin{array}{cll}
                \max & r_a \sum\limits_{i \in \bar{\mc I}_a}  \gamma_{1 i} - r_{a'} \sum\limits_{i \in\bar{\mc I}_{a'}} \left(\frac{1}{N} -\gamma_{4 i}\right)  \\
                \st &  \gamma_{1} \in \R_+^{|\bar{\mc I}_a|},\; \gamma_2 \in \R_+^{|\bar{\mc I}_a|},\; \gamma_3 \in \R_+^{|\bar{\mc I}_{a'}|},\; \gamma_4 \in \R_+^{|\bar{\mc I}_{a'}|}
                \\[1.5ex]
                & \rho - \sum\limits_{i \in \mc   I_a} \gamma_{1 i} d_{1i} - \sum\limits_{i \in \mc  I_{a'}}  \gamma_{4 i} d_{0i} \geq 0 \\
                & \gamma_{1 i} +\gamma_{2 i} =  {1}/ {N} & \forall i \in \bar{\mc I}_{a}\\
                & \gamma_{3 i} + \gamma_{4 i} = {1}/{N} &\forall i \in \bar{\mc I}_{a'}.
            \end{array}
        \ee
    Because the variables $\gamma_{2i} \in \R_+^{|\bar{\mc I}_a|}$ and $\gamma_{3 i} \in \R_+^{|\bar{\mc I}_{a'}|} $ do not appear in the objective of problem~\eqref{eq:quan_kappa_inf_dual-2} and $r_a, r_{a'} > 0$, we can further simplify problem~\eqref{eq:quan_kappa_inf_dual-2} to
                 \be
        \label{eq:quan_kappa_inf_dual-3}
            \begin{array}{cll}
                \max & r_a \sum\limits_{i \in \bar{\mc I}_a}  \gamma_{1 i} - r_{a'} \sum\limits_{i \in\bar{\mc I}_{a'}} \left(\frac{1}{N} -\gamma_{4 i}\right)  \\
                \st & \gamma_{1} \in \R_+^{|\bar{\mc I}_a|}, \gamma_4\in \R_+^{|\bar{\mc I}_{a'}|}
                \\
                & \rho - \sum\limits_{i \in \mc   I_a} \gamma_{1 i} d_{1i} - \sum\limits_{i \in \mc  I_{a'}}  \gamma_{4 i} d_{0i} \geq 0 \\
                & \gamma_{1 i} \leq {1}/ {N} & \forall i \in \bar{\mc I}_{a}\\
                & \gamma_{4 i} \leq {1}/{N} &\forall i \in \bar{\mc I}_{a'} .
            \end{array}
        \ee
        Note that for $\gamma_1^\star$ and $\gamma_4^\star$ that optimize problem \eqref{eq:quan_kappa_inf_dual-3} for all $i\notin {\mc I}_a$, $\gamma^\star_{1i}$ takes the value $1/N$, and similarly for all $i\notin {\mc I}_{a'}$, $\gamma^\star_{4i}$ takes the value $1/N$.  
        Hence, it is sufficient to optimize over values of $\gamma_{1i}$ and $\gamma_{4i}$ for all $i \in \mc I_a \cup \mc I_{a'}$.
         By applying the variable transformations  $\gamma_{1i} \xleftarrow[]{} z_i/N $ for all $i \in \mc I_{a}$ and $ \gamma_{4 i} \xleftarrow[]{} z_i/N $ for all $i \in \mc I_{a'}$, where $z \in \R_+^{|\mc I_{a}|+|\mc I_{a'}|}$, the problem \eqref{eq:quan_kappa_inf_dual-3} can be restated as 
    \be
        \label{eq:quan_kappa_inf_dual-4}
            \begin{array}{cll}
                \max & r_a \frac{| \bar{\mc I}_a \backslash \mc I_a|}{N} + \frac{r_a}{N}\sum\limits_{i \in \mc I_a}z_i - r_{a'} \frac{|\mc I_{a'}|}{N} + \frac{r_{a'}}{N}\sum\limits_{i \in \mc I_{a'}}z_i \\
                \st & z \in \R_+^{|\mc I_{a}|+|\mc I_{a'}|}
                \\[1.5ex]
                & \sum\limits_{i \in \mc  I_a} z_i d_{1i} + \sum\limits_{i \in \mc  I_{a'}} z_i d_{0i} \leq N\rho\\
                & z_i \leq 1 & \forall~i \in \mc I_{a} \cup \mc I_{a'}.
            \end{array}
        \ee
        Observe that $r_a |\bar{\mc I}_a\backslash \mc I_a |/N -r_{a'} |\mc I_{a'}|/N$ is equivalent to empirical value function $\hat \VV(a, a')$, which is defined as in the theorem statement.
    By introducing the non-negative rewards and weights through
    \[
    (c_{aa'i}, w_{aa'i}) =\begin{cases}
    (r_a, d_{1i})\quad &\text{if}~i \in \mc I_a,\\
    (r_{a'}, d_{0i}) \quad &\text{if}~i \in \mc I_{a'},\\
    (0, +\infty) & \text{otherwise},
    \end{cases}
    \]
    we can re-write the optimization problem in \eqref{eq:quan_kappa_inf_dual-4} as
    \[\hat{\mathds V}(a,a') + \max\limits_{z \in [0,1]^N} \left\{\frac{1}{N} \sum\limits_{i\in[N]} c_{aa'i} z_i~:~\frac{1}{N}\sum\limits_{i \in[N]} w_{aa'i}z_i \leq \rho\right\}~\forall a,a' \in \mc A, a\neq a',\]
    where the equivalence of the two problems holds because $z_i\opt=0$ for all $\{i \in [N] : w_{aa'i} = +\infty\}$.
     This observation concludes the proof.
    \end{proof} 
    \begin{proof}[Proof of Proposition~\ref{prop:extreme-infty}]
    For $\rho = 0$, we have $\mathds V(a, a') = \hat {\mathds V}(a, a')$ and $\QQ\opt = \Pnom_N$ is the optimal solution that attains the supremum in~\eqref{eq:V_f-def}. For the rest of the proof, it suffices to consider when $\rho > 0$.
    
    We define the set $\mc I = \{i\in[N]:\hat a_i = a, \hat y_i =1,  \hat x_i \in \text{int}(\mc X_0)\} \cup \{i\in [N] : \hat a_i=a', \hat y_i =1, \hat x_i \in \text{int}(\mc X_1)\}$.
    First, we show that $\QQ\opt$ defined in the statement of the Proposition~\ref{prop:extreme-infty} satisfies $\QQ\opt\in \mbb B_\rho(\Pnom_N)$. Notice that $\QQ\opt$ does not flip any label on $A$ and $Y$ as $\kappa_{\mc A} = \kappa_{\mc Y} = \infty$, thus it preserves the marginals
        \[
            \QQ\opt(A =  a, Y =  y) = \Pnom_N(A =  a, Y =  y) \quad \forall a \in \mc A, y \in \mc Y.
        \]
        Moreover, the distance from $\QQ\opt$ to $\Pnom_N$ satisfies
        \begin{align*}
            \Wass(\QQ\opt, \Pnom_N) &\le \frac{1}{N} \sum_{i \in [N]} z_i\opt \| \wh x_{i}\opt - \wh x_i \| = \frac{1}{N} \sum_{i \in\mc I} w_{aa'i} z_i\opt \le \rho,
        \end{align*}
        where the first inequality follows by definition of the Wasserstein distance, the equality is from the definition of $w_{aa'i}$, and the last inequality is from the feasibility of $z\opt$ in the linear program in~\eqref{eq:quantify-EO-infinite-fin}.
    
        In what follows, we will construct a distribution $\QQ\opt_\eps \in \mathbb{B}_\eps(\QQ\opt)$ that is $\eps$-suboptimal in \eqref{eq:quantify-EO-infinite-fin} for $\rho > 0$.
        For simplicity of exposition, we assume that $\hat x_i \neq \hat x_i\opt$ for all $i \in [N]$ and the norm on $\mc X$ used in the Wasserstein ground metric is a 2-norm.
        For any given $ \eps$, we choose $\theta \in [0,1]$ that satisfies
        $\theta \ge  1 - {N\eps}/{\sum_{\substack{i \in \mc J}} r_{a'} z_i\opt}$, where $\mc J = \{i \in [N]:\wh x_i \in  \mc X_1, \wh a_i = a', \wh y_i = 1\}$.
        We set $\epsilon_0,\epsilon_1, \epsilon_2 \in \R_+$ to satisfy the following criteria
        \[
            \left\{
                \begin{array}{l}
                    \theta \epsilon_0 +(1-\theta) \epsilon_1 \le \eps, \\
                    (1-\theta)(\epsilon_1 + \epsilon_2) \ge \eps
                \end{array}
            \right.
        \]
        so that for all $i \in [N]$, the set
        \[
            \left\{ x \in \mc X_1: \| x - \wh x_i\opt \| \le \epsilon_1, \| x - \wh x_i\| \le \| \wh x_i\opt - \wh x_i \| - \epsilon_2 \right\}
        \]
        is non-empty.  When the norm on $\mc X$ is a 2-norm, the above condition is satisfied by setting $\theta \epsilon_0 = \eps /2$, $(1-\theta) \epsilon_1 = \eps/2$, and $\epsilon_2 = \epsilon_1$. For other norms, this requirement can be satisfied by properly scaling $\epsilon_0$ down and scaling $ \epsilon_1$ and $\epsilon_2$ up to meet the criteria.
        For each $i \in [N]$, consider the tuple $(\wh x_{0i}^\eps, \wh x_{1i}^\eps)$ defined as
        \[
             (\wh x_{0i}^\eps, \wh x_{1i}^\eps)
             = 
             \begin{cases}
                (\wh x_{0i}, \wh x_{1i}) & \text{if } i \in \mc J, \\
                (\wh x_{i}\opt, \wh x_i\opt) & \text{otherwise},
            \end{cases}
        \]
        where $\wh x_{0i} \in \mc X_0$ such that $\| \wh x_{0i} - \wh x_i\opt\|\le \epsilon_0$, and $\wh x_{1i} \in \mc X_1$ such that $\| \wh x_{1i} - \wh x_i\opt \|\le \epsilon_1$, and $\| \hat x_{1i} - \hat x_i\|\leq\|\hat x_i\opt- \hat x_i\|- \epsilon_2$. Notice that the existence of $\wh x_{0i}$ is guaranteed because $\wh x_{i}\opt$ is the projection of $\wh x_i$ onto $\partial \mc X_1$, or equivalently onto $\mathrm{cl}(\mc X_0)$, and hence $\mc X_0 \cap \{ x_i: \| x_i - \wh x_{i}\opt \| \le \epsilon_0\} $ is non-empty for any $\epsilon_0 \in \R_{++}$.
        Consider now distribution $\QQ\opt_\eps$ that is constructed as
       \be \notag
            \QQ\opt_\eps=\frac{1}{N} \left( 
            \textstyle \sum\limits_{i=1}^N \theta z_i\opt \delta_{(\wh x_{0i}^\eps, \wh a_i, \wh y_i)} + \sum\limits_{i=1}^N z_i\opt(1-\theta) \delta_{(\wh x_{1i}^\eps, \wh a_i, \wh y_i)} + \sum\limits_{i=1}^N (1- z_i\opt) \delta_{(\wh x_i, \wh a_i, \wh y_i)}
        \right).
        \ee
        We will show that $\QQ\opt_\eps \in \mathbb{B}_\rho(\QQ\opt)$.
        By definition of $\QQ\opt_\eps$, we have
        \begin{align*}
            \Wass(\QQ\opt_\eps, \QQ\opt) &\le \frac{1}{N} \sum\limits_{i \in \mc J} \big( \theta z_i\opt \| \wh x_{0i} - \wh x_{i}\opt \| + (1-\theta) z_i\opt  \|\hat x_i\opt - \hat x_{1i}\| \big)\\
            &\le \theta \epsilon_0 + (1-\theta)\epsilon_1 \le \eps,
        \end{align*}
        where the first inequality is due to $z_i\opt \leq 1$ for all $i \in [N]$, $\mc J \subset [N]$, $\|\hat x_{0i}- \hat x_i\opt\| \leq \epsilon_0$ and $\| \hat x_{1i} - \hat x_{i}\opt\| \leq \epsilon_1$.
        The last inequality follows by assumption on $\epsilon_0$ and $\epsilon_1$.
        Next, we show that $\QQ_\eps \in \mathbb{B}_\rho(\Pnom_N)$.
        Similarly, by construction of $\QQ_\eps\opt$ we have
        \begin{align*}
            \Wass(\QQ\opt_\eps, \Pnom_N) & \le \frac{1}{N} \left( \sum_{i \in [N]}\theta z_i\opt\| \wh x_{0i}^\eps - \wh x_{i} \|  + \sum_{i \in [N]} (1-\theta)z_i\opt \| \wh x_{1i}^\eps - \wh x_i\| \right)\\
            &=\frac{1}{N} 
            \sum\limits_{i \in [N]\backslash \mc J} \theta z_i\opt \|\hat x_{i}\opt - \hat x_i\| +\frac{1}{N} \sum\limits_{i \in \mc J}\theta z_i\opt \|\hat x_{0i} - \hat x_i\| \\
            &\hspace{1cm}+ \frac{1}{N}\sum\limits_{i\in [N]\backslash \mc J}(1-\theta) z_i\opt \|\hat x_i\opt - \hat x_i\|+\frac{1}{N}\sum\limits_{i \in \mc J}(1-\theta) z_i\opt\|\hat x_{1i} - \hat x_i\|
            \\
            &\le \frac{1}{N} 
            \sum\limits_{i \in [N]\backslash \mc J} \theta z_i\opt \|\hat x_{i}\opt - \hat x_i\| +\frac{1}{N} \sum\limits_{i \in \mc J}\theta z_i\opt \|\hat x_{0i} - \hat x_i\| \\
            &\hspace{1cm}+ \frac{1}{N}\sum\limits_{i\in [N]\backslash \mc J}(1-\theta) z_i\opt \|\hat x_i\opt - \hat x_i\|+\frac{1}{N}\sum\limits_{i \in \mc J}(1-\theta) z_i\opt\|\hat x_i\opt - \hat x_i\| - (1-\theta)\epsilon_2
            \\
            &=\frac{1}{N} 
            \sum\limits_{i \in [N]}  z_i\opt \|\hat x_{i}\opt - \hat x_i\| +\frac{1}{N} \sum\limits_{i \in \mc J}\theta z_i\opt (\|\hat x_{0i} - \hat x_i\| - \|\hat x_i\opt - \hat x_i\|)- (1-\theta)\epsilon_2
            \\
            &\leq\rho +\frac{1}{N} \sum\limits_{i \in \mc J}\theta z_i\opt \|\hat x_{0i} - \hat x_i\opt\| -(1-\theta) \epsilon_2 \leq \rho + \theta \epsilon_0 -(1-\theta)\epsilon_2\le  \rho,
        \end{align*}
        where the first equality is due to the definition of $\hat x_{0i}^\eps$ and $\hat x_{1i}^\eps$.
        The second inequality follows by construction of $\hat x_{1i}$, that is, it satisfies $\|\hat x_{1i} - \hat x_i \|\leq\|\hat x_i\opt - \hat x_i\|- \epsilon_2$.
        The third inequality follows from triangle inequality, that is, $\|\hat x_{0i} - \hat x_i\| \le   \| \hat x_{0i} - \hat x_i\opt\| + \|\hat x_i - \hat x_i\opt \| $ and since $z_i\opt$ is feasible in \eqref{eq:quantify-EO-infinite-fin}.
        The last equality is due to the choice of $\epsilon_0$ and $\epsilon_2$ that satisfies $\theta \epsilon_0 + (1-\theta)\epsilon_2 \le 0$. As a consequence, we have $\QQ\opt_\eps \in \mbb B_{\rho}(\Pnom_N)$.   
        
    In the last step, we verify that  $\QQ_\eps\opt$ is an $\eps$-suboptimal solution of the maximization problem that defines $\mathds V(a, a')$. Notice that because $\Pnom_N$ is an empirical distribution, we have
    \[
        \hat{\mathds{V}}(a, a') = \frac{1}{N} \sum\limits_{\substack{i \in [N] : \,\hat x_i \in \mc X_1\\ \hat a_i = a,\,\hat y_i = 1}}r_a - \frac{1}{N} \sum\limits_{\substack{i \in [N] : \,\hat x_i \in \mc X_1\\ \hat a_i = a',\,\hat y_i = 1}}r_{a'}.
    \]
   By definition of $\QQ\opt_\eps$, we have the following equalities
    \begin{subequations}
      \begin{align}
      \label{eq:extrem_unf1}
        \QQ\opt_\eps(X \in \mc X_1 | A = a, Y = 1) &= \QQ\opt(X \in \mc X_1 | A = a, Y = 1) \\
    \label{eq:extrem_unf2}
        \QQ\opt_\eps(X \in \mc X_1 | A = a', Y = 1) &= 
        \QQ\opt(X \in \mc X_1 | A = a', Y = 1) - \frac{\theta}{N}\sum\limits_{i \in \mc J} r_{a'} z_i\opt
    \end{align}
    \end{subequations}
Similarly by definition of $\QQ\opt$, we have the following equalities
        \begin{align}
        &{\QQ\opt} [X \in \mc X_1| A=a, Y=1] - {\QQ\opt} [X \in \mc X_1| A=a', Y=1]\nonumber\\ 
        &\hspace{0.1cm}=\frac{1}{N}\sum\limits_{\substack{i \in [N]: \hat a_i =a
        \\ \hat y_i =1}}r_{a}z_i\opt + \frac{1}{N}\sum\limits_{\substack{i \in [N]:\hat x_i \in \mc X_1\\ \hat a_i =a, \hat y_i =1}}r_{a}(1-z_i\opt) - \frac{1}{N} \sum\limits_{\substack{i \in [N]: \hat a_i =a'\\ \hat y_i =1}}r_{a'}z_i\opt -\frac{1}{N} \sum\limits_{i \in \mc J}r_{a'}(1-z_i\opt)\nonumber\\
        &\hspace{0.1cm}=\hat{\mathds V}(a,a')+ \frac{1}{N}\sum\limits_{\substack{i \in [N]: \hat x_i \in \text{int}(\mc X_0) \\\hat a_i =a, \hat y_i =1}}r_{a}z_i\opt  -\frac{1}{N} \sum\limits_{\substack{i \in [N]: \hat x_i \in\text{int}(\mc X_0),\\\hat a_i =a', \hat y_i =1}}r_{a'}z_i\opt\nonumber \\
        &\hspace{0.2cm}=\hat{\mathds V}(a,a')+ \frac{1}{N}\sum\limits_{\substack{i \in [N]: \hat x_i \in \text{int}(\mc X_0) \\\hat a_i =a, \hat y_i =1}}r_{a}z_i\opt  + \frac{1}{N} \sum\limits_{\substack{i \in [N]:\hat x_i \in \text{int}(\mc X_1)\\\hat a_i =a', \hat y_i = 1}}r_{a'} z_i^\star- \frac{1}{N} \sum\limits_{\substack{i \in [N]:\hat x_i \in \text{int}(\mc X_1)\\\hat a_i =a', \hat y_i = 1}}r_{a'} z_i^\star \nonumber\\
        &\hspace{0.1cm}=\mathds V(a,a') - \frac{1}{N} \sum\limits_{\substack{i \in [N]: \hat x_i \in  \text{int}(\mc X_1) \\ \hat a_i = a', \hat y_i =1}}r_{a'}z_i\opt,\hspace{-.2cm}\label{eq:opt_dist_unf}
        \end{align}
        where the first equality follows by construction of $\QQ\opt$, and the second equality follows from the definition of $\hat{\mathds V}(a,a')$.
        The third equality follows by realizing that $z_i\opt = 0$ for all indices in the set $\{i\in[N] : \hat x_i \in \text{int}(\mc X_0), \hat a_i =a', \hat y_i =1\}$, and we add and subtract the same term to have a representation in terms of $\mathds V(a,a')$.
        Moreover, the last equality is due to the definition of $\mathds V(a,a')$.
    
    Now, we will show that $\QQ_\eps\opt$ provides $\eps$-suboptimal solution to the maximization problem that defines $\mathds V(a,a')$. By taking the difference of \eqref{eq:extrem_unf1} and \eqref{eq:extrem_unf2},
    \begin{align*}
        &\QQ\opt_\eps(X \in \mc X_1 | A = a, Y = 1) - \QQ\opt_\eps(X \in \mc X_1 | A = a', Y = 1)\\
         &=\QQ\opt(X \in \mc X_1 | A = a, Y = 1) - \QQ\opt(X \in \mc X_1 | A = a', Y = 1)+\frac{\theta}{N}\sum\limits_{i \in \mc J} r_{a'}  z_i\opt \\
         &=\mathds V(a,a') - \frac{1}{N} \sum\limits_{i \in \mc J}r_{a'}z_i\opt + \frac{\theta}{N}\sum\limits_{i \in \mc J} r_{a'}  z_i\opt
        \geq \mathds{V}(a,a') -\eps,
    \end{align*}
     where the second equality is due to \eqref{eq:opt_dist_unf}.
     The last inequality follows as 
     $\theta \ge  1 - {N\eps}/{\sum_{\substack{i \in \mc J}} r_{a'} z_i\opt}$.
    This concludes the proof.
     \end{proof}
\subsection{Additional Theoretical Results}
\label{app:additional_res}

In the main paper, we solve problem in \eqref{eq:dro_fair_training} for general $\kappa_\mc{A}$ and $\kappa_{\mc Y}$.
If $\kappa_{\mc A}$ and $\kappa_{\mc Y}$ ceases to be finite then the problem can be substantially simplified.
\begin{corollary}
[Absolute trust in $A$ and $Y$]
    \label{corol:train_inf_kappa_fin}
    If $f(z) = \log(z)$, $\eta \leq \min\{\hat p_{11},\hat p_{01}\}$ and $\kappa_{\mc A} = \kappa_{\mc Y} = \infty$, then problem~\eqref{eq:dro_fair_training} simplifies to the following tractable convex program
    \begin{equation*}
     \label{eq:train_inf_kappa_fin}    \begin{array}{cll}
        \min & t\\
        \st& \beta \in \R^p ,\; t \in \R,\; \lambda_0,\; \lambda_1 \in \R_+,\;\nu_0, \nu_1 \in \R^N \\
        & \hspace{-2mm} \left.
        \begin{array}{l}
        \| \beta \|_*(1 + \eta r_{a'}) \leq \lambda_a\\
        \rho \lambda_a  + \frac{1}{N} \sum_{i =1}^N \nu_{ai} \leq t \\
        \hspace{-2mm} \left. \begin{array}{ll}
        \nu_{ai} +\log(h_\beta(-\hat x_i)) \geq 0& \text{ if}~\hat y_i = 0\\
        \nu_{ai} + (1 - \eta r_a) \log(h_\beta(\hat x_i)) \ge 0& \text{ if}~\hat a_i = a,\phantom{'}\; \hat y_i = 1\\
        \nu_{ai} + (1+\eta r_{a'})\log(h_\beta(\hat x_i)) \ge 0& \text{ if}~\hat a_i = a',\; \hat y_i = 1\\
        \end{array}
        \right\} \,\forall i \in [N]\end{array}\right\}~ 
        \begin{array}{l}
        \forall a,a' \in \mc A:\\
        a'=1-a.
        \end{array}
        \end{array}
    \end{equation*}
    \end{corollary}
\begin{proof}[Proof of Corollary~\ref{corol:train_inf_kappa_fin}]
    The proof follows the same steps as the proof of Theorem~\ref{thm:training_refor} until the reformulation of $\mathds T_\beta^{aa'}$.
    Thanks to Corollary~\ref{corol:conic-duality-inf-kappa}, $\mathds T_\beta^{aa'}$ coincides with the optimal value of
    \be
        \label{eq:dual-inf-1}
        \begin{array}{cll}
        \inf & \rho \lambda + \frac{1}{N} \sum_{i =1}^N \nu_i \\
        \st & \lambda \in \R_+,~\nu \in \R^N \\
        & \lambda \|x_i - \hat{x}_i\|+  \nu_i \ge  \phi^{aa'}_\beta(x_i,\hat  a_i, \hat y_i) ~ \forall x_i \in \mc X ,\;\forall i \in [N],
        \end{array}
    \ee
    where the function $\phi_\beta^{aa'}$ is as it is defined in the proof of Theorem~\ref{thm:training_refor}.  
    We now proceed to consider the constraint of problem~\eqref{eq:dual-inf-1}, which can be written in a simplified form as
    \be \label{eq:dual-constraint-inf}
        \nu_i \ge \Sup{x_i \in \mc X} \left\{ \phi^{aa'}_\beta(x_i, \hat a_i, \hat y_i) - \lambda \| x_i - \wh x_i \| \right\}.
    \ee
    Suppose that $\hat y_i = 0$, then $\phi_\beta^{aa'}(x_i, \hat a_i, 0) = -\log(1-h_\beta(x_i))$, and by Lemma~\ref{lemma:conjugacy_reform}, we have
    \begin{align*}
        \Sup{x_i \in \mc X} \left\{ \phi_\beta^{aa'}(x_i, a_i, 0) - \lambda \| x_i - \wh x_i \| \right\} &= \Sup{x_i \in \x} -\log(1-h_\beta(x_i))-\lambda\| x_i - \hat x_i\| \\
        &= 
        \begin{cases}
        -\log(1-h_\beta( \hat x_i)) \quad &\text{if}~ \|\beta\|_* \leq \lambda, \\
        +\infty &\text{otherwise},
        \end{cases}
    \end{align*}
    and so the constraint~\eqref{eq:dual-constraint-inf} when $\hat y_i = 0$ becomes
    \[
    \left\{
            \begin{array}{l}
                \nu_i \ge -\log(1-h_\beta(\hat x_i))\\
                \|\beta\|_* \leq \lambda.            \end{array} \right.
    \]
    If $\hat a_i = a$ and $\hat y_i = 1$, then $\phi^{aa'}_\beta(x_i, a, 1) = (\eta r_a - 1) \log(h_\beta(x_i))$. 
    We thus have by Lemma~\ref{lemma:conjugacy_reform} that
    \[
        \Sup{x_i \in \mc X} \left\{ \phi^{aa'}_\beta(x_i, a, 1) - \lambda \| x_i - \wh x_i \| \right\} =
        \begin{cases}
        (\eta r_a - 1) \log(h_\beta(\wh x_i)) \quad &\text{if}~  (1 -\eta r_a)\|\beta\|_* \leq \lambda, \\
        +\infty &\text{otherwise}.
        \end{cases}
    \]
    If $\hat a_i = a$ and $\hat y_i = 1$, then the constraint~\eqref{eq:dual-constraint-inf} becomes
    \[
        \left\{
            \begin{array}{l}
                \nu_i \ge (\eta r_a -1) \log(h_\beta(\wh x_i)) \\
                (1 -\eta r_a)\|\beta\|_* \le \lambda.
            \end{array}
        \right.
    \]
    Using an analogous argument for the case where $\hat a_i = a'$ and $\hat y_i = 1$, we have $\phi^{aa'}_\beta(x_i, a', 1) = -(1+ \eta r_{a'}) \log(h_\beta(x_i))$. By Lemma~\ref{lemma:conjugacy_reform}, we have
    \[
        \Sup{x_i \in \mc X} \left\{ \phi_\beta^{aa'}(x_i, a', 1) - \lambda \| x_i - \wh x_i \| \right\} =
        \begin{cases}
        -(1 +\eta~ r_{a'}) \log(h_\beta(\wh x_i)) \quad &\text{if}~  (1 +\eta~ r_{a'})\|\beta\|_* \leq \lambda, \\
        +\infty &\text{otherwise}.
        \end{cases}
    \]
    If $\hat a_i = a'$ and $\hat y_i = 1$, then the constraint~\eqref{eq:dual-constraint-inf} is equivalent to
    \[
        \left\{
            \begin{array}{l}
                \nu_i \ge - (1 + \eta~ r_{a'}) \log(h_\beta(\wh x_i))\\
                (1 + \eta~ r_{a'})\|\beta\|_* \le \lambda.
            \end{array}
        \right.
    \]
      Injecting all the specific cases of constraint~\eqref{eq:dual-constraint-inf} into problem~\eqref{eq:dual-inf-1}, the value $\mathds T_\beta^{a a'}$ is equal to the optimal value of the following optimization problem
     \be
        \label{eq:prob_dual-inf-2}
        \begin{array}{cll}
        \min & \rho \lambda  + \frac{1}{N} \sum_{i =1}^N \nu_i \\
        \st & \lambda \in \R_+,~\nu \in \R^N \\
        & \|\beta\|_* \leq \lambda , ~ \| \beta \|_*(1 - \eta r_a) \leq \lambda,~
         \| \beta \|_*(1 + \eta r_{a'}) \leq \lambda\\
        &\hspace{-2mm} \left. \begin{array}{ll}
        \nu_i \geq -\log(1- h_\beta(\hat x_i)) & \text{ if}~\hat y_i = 0\\
        \nu_i \geq (\eta r_a-1) \log(h_\beta(\hat x_i)) & \text{ if}~\hat a_i = a, \hat y_i = 1\\
        \nu_i \ge -(1+\eta r_{a'})\log(h_\beta(\hat x_i))& \text{ if}~\hat a_i = a', \hat y_i = 1\\
        \end{array}
        \right\} ~\forall i \in [N].
        \end{array}
    \ee

        Note that the constraints $\|\beta\|_* \leq \lambda$ and $\|\beta \|_* (1- \eta r_a) \leq \lambda$ are redundant in view of the constraint $\|\beta\|_* (1+\eta r_{a'}) \leq \lambda$. The claim then follows by substituting the dual reformulations for $\mathds T_\beta^{aa'}$ into~\eqref{eq:epigraph-for} and eliminating the embedded minimization operators.
    \end{proof}

	 \section{Further Discussion and Details of Numerical Results}
    \label{sect:further_disc_numerical}
    In this section, we provide further details about the experiments in the Section~\ref{sec:numerical}, including synthetic experiments, real dataset experiments and illustrations of the extremal distribution.
    All optimization problems are implemented in Python 3.7 and all experiments were run on an Intel i7-700K CPU (4.2 GHz). 
    
    \textbf{Synthetic Experiments.} 
To show the decision boundaries in Figure~\ref{fig:decision_boundaries_frontiers}, we generate binary classification data that has 2 dimensional feature vectors with two subgroups one of them being the minority (i.e., $A=0$).
We generate 5000 and 2000 binary class labels $Y \in \{0, 1\}$ uniformly at random for majority subgroup ($A=1$), and minority subgroup ($A=0$) respectively.
Then, we set the conditional true distributions of 2 dimensional feature vectors as following Gaussian distributions.
\begin{align*}
    X|A = 1, Y = 1 &\sim \mc N([6, 0], [3.5, 0; 0, 3.5]),\\
X|A = 1, Y = 0 &\sim \mc N ([2, 0], [3.5, 0; 0, 3.5]),\\
X|A = 0, Y = 0 &\sim \mc N ([-4, 0], [5, 0; 0, 5]),\\
X|A = 0, Y = 1 &\sim \mc N ([-2, 0], [5, 0; 0, 5]).
\end{align*}
Next, we use stratified sampling\footnote{Stratified sampling is a method of sampling from a population which can be partitioned into subgroups, and requires sampling each subgroup independently. }
to obtain $N=50$ points from the generated data as a training dataset.
We set the rest of the dataset the test dataset that we calculate the accuracy and the unfairness of the trained models.

         To obtain the Pareto frontiers in Figure~\ref{fig:decision_boundaries_frontiers}, we use the synthetic experiment from~\cite{ref:zafar2015fairness}.
        In this setting, we set the true distributions of the class labels $\PP(Y=0) = \PP(Y=1)= 1/2$.
        Next, we set the conditional distributions of the 2 dimensional feature vectors as the following Gaussian distributions
        \begin{align*}
            X|Y=1 \sim \mc N([2; 2], [5, 1; 1, 5]),\,
            X|Y=0 \sim \mc N([-2; -2], [10, 1; 1, 3]).
        \end{align*}
        Then, we draw sensitive attribute of each sample $x$ from a Bernoulli distribution,
        \[\PP(A = 1 | X = x') = pdf(x' |Y = 1)/(pdf(x'|Y = 1) + pdf(x'|Y = 0)),\] where $x'= [\cos(\pi/4),  \sin(\pi/4); \sin(\pi/4), \cos(\pi/4)]x$ is a rotated version of the feature vector $x$ and $pdf(\cdot | Y = y)$ is the Gaussian probability density function of $X | Y= y$.
        
    We sample 400 i.i.d.~samples from $\PP$ as our dataset, and we stratify sample 100 data points from this dataset and set it as training set, while we set the rest as the test dataset. The procedure to obtain the frontiers is explained as in Section~\ref{sec:numerical}. We fix $\rho$ for DR-FLR to 0.01 and and the range of $\eta$ is $[10^{-4}, \min \{\hat p_{11}, \hat p_{01}\}]$ with 5 equi-distant points.
    
    \noindent\textbf{Experiments with Real Data.} We consider four publicly available datasets (Adult, Drug, COMPAS, Arrythmia).
     We obtain Adult dataset from UCI repository\footnote{https://archive.ics.uci.edu/ml/datasets/adult}, it contains 14 features concerning demographic characteristics of 45222 instances (32561 for training and 12661 for test). The prediction task is to determine whether a person makes over 50000$\$$ a year, where we consider \textit{gender} as the sensitive attribute.
The Drug dataset\footnote{https://archive.ics.uci.edu/ml/datasets/Drug+consumption+$\%$28quantified$\%$29} have records for 1885 respondents.
Each respondent is described by 12 features, including level of education, age, gender, country of residence and ethnicity. 
The task is to determine whether the user ever used heroin or not.
We consider \textit{ethnicity} as the sensitive attribute.
COMPAS (Correctional Offender Management Profiling for Alternative Sanctions)\footnote{https://www.propublica.org/datastore/dataset/compas-recidivism-risk-score-data-and-analysis} is a popular
algorithm used by judges and parole officers for scoring criminal
defendant’s likelihood of recidivism.
It has been shown that the algorithm
is biased in favor of white defendants based on a 2 year follow up study. 
This dataset contains variables used by the COMPAS algorithm in scoring defendants, along with their
outcomes within 2 years of the decision for over 10000 criminal defendants.
We concentrate on the one 
that includes only violent recidivism, where \textit{ethnicity} is the sensitive attribute.
We obtain the Arrhythmia dataset from UCI repository\footnote{https://archive.ics.uci.edu/ml/datasets/Arrhythmia} which contains 279 attributes\footnote{ We only use the first 12 out of 278 non-sensitive features of the Arrhythmia dataset so that we can use the same search grid for $\rho$ across all datasets (in the other datasets $p$ ranges from 5 to 12).}, where the aim is to distinguish between the presence and absence of cardiac arrhythmia and to classify it in one of the 16 groups.
In our case, we changed the task with the binary classification between normal arrhythmia against 15 different classes of arrhythmia.

\noindent\textbf{Training, Validation and Testing Procedure.}
In all other datasets we randomly select $2/3$ of the samples for training and we set the rest of the data for testing. We repeat the training, validation and testing process for $K_3$ times, while the Adult dataset comes with designated training and testing samples, and thus $K_3=1$. 

\noindent\textbf{Validation.} We select the hyper-parameter(s) of the classifier(s) (e.g., the radius of the Wasserstein ball for DR-FLR) using a cross-validation procedure on the training set similar to~\cite{ref:donini2018empirical}.
First, we collect statistics of the parameters of the model by splitting the training set into sub-training set ($N$ samples) and a validation set for $K_1$ times.
In the first step, the value of the parameter in the grid with highest accuracy calculated over the validation set is identified.
In the second step, we shortlist all the values of parameter in the grid with accuracy close (in our case $70 \% - 98\%$) to the maximum accuracy in that range minus the lowest possible accuracy. 
Finally, from this list, we select the parameter value that provides the lowest unfairness measure with respect to the log-probabilistic equalized opportunity.

\noindent\textbf{Testing.} We stratify sample $N$ samples from the training set and we collect the statistics regarding the performance of the classifiers on the test dataset.
We repeat this process for $K_2$ times.

\noindent\textbf{Discussion on Table~\ref{tab:results} in Section~\ref{sec:numerical}.} Table~\ref{tab:results} summarizes the testing accuracy and unfairness of averaged over $K_1=3, K_2=100, K_3=2$, where we tune the radius of Wasserstein ball $\rho \in [10^{-5},10^{-1}]$\footnote{After we obtain the logarithmic scale, we multiply the values by 5, and thus $\rho \in [5.10^{-5}, 5.10^{-1}]$ at the end.} for DR-FLR classifier on a logarithmic search grid with 50 discretization points:
All methods are trained with $N=150$ and we set $\eta= \min\{\hat p_{11}, \hat p_{01}\}/2$ both for FLR and DR-FLR, $\kappa_\mc{A}= \kappa_{\mc Y} = 0.5$ for DR-FLR,  DOB${}^+$ \cite{ref:donini2018empirical} (the model parameter $\epsilon= 0$), and ZVRG \cite{ref:zafar2015fairness} (the model parameter $\epsilon = 10^{-4}$).
We use the following accuracy thresholds at the validation step to tune radius of Wasserstein distance for DR-FLR: $95\%$ for Drug and Adult, $97\%$ for Arrhythmia dataset and $73\%$ for COMPAS dataset.
The difference of the threshold is due to the structure of dataset. For example, the COMPAS dataset is mostly categorical (other than one attribute that is numerical) and thus to decrease the unfairness, the threshold that we use in the validation step for the accuracy should be smaller than the one would use for other datasets that consists mostly numerical attributes.
Moreover, the accuracy threshold also depends on the unbalancedness of the dataset, which determines the lowest possible accuracy that is attained when a classifier only predicts $1$ (or $0$) for all samples.

\noindent\textbf{Worst Case Distribution.} To illustrate the extremal distribution $\QQ\opt$ from Proposition~\ref{prop:extreme-infty}, we generate two interleaving half circles, which is a simple toy dataset to visualize binary classification algorithm.
We assign the sensitive attributes of the binary classification data points uniformly at random by setting 2/3 of the data as the majority subgroup and while the rest as the minority subgroup.
We generate 500 samples and split it into training and test sets by $85\%$ and $15\%$ respectively.
Next, we train the classifiers with the training set and calculate the worst-case unfairness $\overline{\mathds U}_f$ for prescribed $\rho$.
The illustrated extremal distribution $\QQ^\star$ in Figure~\ref{fig:quantification} are obtained with radius of the Wasserstein ball $0.02, 0.05, 0.05, 0.01$ for classical logistic regression, support vector machine with RBF kernel, Gaussian process wiht RBF kernel and AdaBoost, respectively.
\subsection{Additional Numerical Experiments}
In this section, we provide additional experiments that we provide to compare performance of different classifiers. 
\paragraph{Discussion on Table~\ref{tab:results_donini_tuned_150}.} An interesting experiment would be to compare the performance of DOB${}^+$ and LR, FLR and DR-FLR, when we also tune the parameter of the classifier that is used in DOB${}^+$.
Since, SVM is a deterministic classifier (we cannot calculate log-probabilistic unfairness), in the cross-validation procedure from the acceptable parameter grid, that provides accuracy higher than the given threshold, we choose the parameter that gives the lowest unfairness with respect to the deterministic equalized opportunity both for DR-FLR and DOB${}^+$.

The results in the Table~\ref{tab:results_donini_tuned_150} summarize the testing accuracy and unfairness averaged over $K_1=5,\, K_2=100,\, K_3=5$, where we tune the radius of Wasserstein ball $\rho \in [10^{-5}, 10^{-1}]$\footnote{After we obtain the logarithmic scale, we multiply the values by 5, and thus $\rho \in [5\cdot 10^{-5}, 5\cdot 10^{-1}]$ at the end.} for DR-FLR classifier and regularization parameter $C \in [10^{-1}, 10^2]$ of linear support vector machine for DOB${}^+$ method on a logarithmic search grid with 50 discretization points.
Next, we keep training sample size $N=150$ for all LR, FLR, DOB$^+$ and DR-FLR.
We use the following accuracy thresholds at the validation step to tune $\rho$ for DR-FLR and $C$ for DOB${}^+$: $95\%$ for Drug, Adult, and Arrhythmia datasets and $70\%$ for COMPAS dataset.
\begin{center}
\vspace{-0.7cm}
\begin{table*}
\setlength{\tabcolsep}{0.05cm}
\tiny
\centering
\begin{tabular}{|l|l|c|c|c|c|c|c|}
\hline
Dataset & Metric
&{LR} 
&{FLR} 
&{DOB${}^+$\cite{ref:donini2018empirical}}
&{DR-FLR}\\
\hline
\hline
\multirow{4}{*}{Drug}
    & Accuracy& 
    $\textbf{0.79} {\pm}\textbf{0.01} $&
    $ \textbf{0.79}\pm\textbf{ 0.01}$ &
    $\textbf{0.79} \pm \textbf{0.01}$ &
    $\textbf{0.79} \pm \textbf{0.01}$\\
    & Det-UNF &
    $ 0.06\pm 0.05$ &
    $ 0.06\pm 0.05$ &
    $ 0.09\pm 0.07$ &
    $\mathbf{0.04} \pm \mathbf{ 0.04}$\\
    & Prob-UNF &
    $0.06 \pm 0.05$ &
    $ 0.06\pm 0.05$ &
    - & 
    $ \mathbf{0.05}\pm \mathbf{0.04}$\\
    & LogProb-UNF &
    $ 0.21\pm 0.20$ &
    $0.20 \pm 0.20$ & 
    - & 
    $ \mathbf{0.16} \pm \mathbf{0.14}$\\ 
    \hline
    \hline
\multirow{4}{*}{Adult}
    &Accuracy& 
    $\mathbf{ 0.80} {\pm} \mathbf{0.01}$ &
    $\mathbf{0.80} \pm \mathbf{0.01}$ & 
    $0.79 \pm 0.01$ & 
    $0.79 \pm 0.01$\\
    &Det-UNF &
    $ 0.08 \pm 0.06$ &
    $ \mathbf{0.06}\pm \mathbf{0.06}$ & 
    $0.16 \pm 0.10$ & 
    $\mathbf{0.06}\pm\mathbf{ 0.06}$\\
    & Prob-UNF & 
    $ 0.17 \pm 0.08$ &
    $\mathbf{0.12} \pm\textbf{ 0.08}$ & 
    $-$ & 
    $\mathbf{0.12} \pm \mathbf{0.08}$ \\
    & LogProb-UNF &
    $ 1.01 \pm 0.77$ &
    $0.68 \pm 0.68$ & 
    $-$ & 
    $\textbf{0.64} \pm \mathbf{0.65}$\\ 
\hline
\hline
\multirow{4}{*}{Compas}&Accuracy& 
    $\textbf{0.65}{\pm} \textbf{0.01}$ &
    $\textbf{0.65} \pm \textbf{0.02}$ & 
    $0.60\pm 0.03$ & 
    $0.60 \pm 0.03$ \\
    &Det-UNF  & 
    $ 0.24 \pm 0.04$ &
    $0.23 \pm 0.04$ & 
    $ 0.17\pm 0.06$ & 
    $ \mathbf{0.15}\pm \mathbf{0.07}$ \\
    & Prob-UNF & 
    $0.12 \pm 0.02$ &
    $ 0.10\pm 0.03$ & 
    $-$ & 
    $ \textbf{0.03}\pm \textbf{0.02}$ \\
    & LogProb-UNF & 
    $ 0.25 \pm 0.06$ &
    $0.22 \pm 0.06$ & 
    $-$ & 
    $\textbf{0.07}\pm \textbf{0.04}$ \\
\hline 
\hline
\multirow{4}{*}{Arrhythmia}&Accuracy& 
    $ 0.63{\pm} 0.03$ &
    $0.63 \pm 0.03$ & 
    $\textbf{0.65}\pm \textbf{0.02}$ & 
    $0.62 \pm 0.03$ \\
    &Det-UNF  & 
    $ 0.21 {\pm} 0.11$ &
    $0.15 \pm 0.10$ & 
    $ 0.11 \pm 0.08$ & 
    $\textbf{0.09} \pm \textbf{0.08}$ \\
    & Prob-UNF & 
    $0.14 {\pm} 0.07$ &
    $ 0.09\pm 0.06$ & 
    $-$ & 
    $\textbf{0.05}\pm \textbf{0.04}$ \\
    & LogProb-UNF & 
    $ 0.28 \pm 0.17$ &
    $ 0.19 \pm 0.15$ & 
    $-$ & 
    $\textbf{0.09} \pm \textbf{0.08}$ \\ 
\hline 
\end{tabular}
\caption{Testing accuracy and unfairness (average $\pm$ standard deviation). For DR-FLR $\rho$ and for DOB$^+$ method regularization parameter $C$ of linear SVM is tuned given the training data.
LR, FLR, DOB${}^+$ and DR-FLR are trained with $N=150$ samples stratify sampled from the training split. }
\label{tab:results_donini_tuned_150}
\end{table*}
\end{center} 

\bibliographystyle{siam} 
\bibliography{references}

\end{document}